\documentclass{article}
\usepackage{iclr2023_conference,times}

\usepackage{amsmath}
\usepackage{amssymb}
\usepackage{amsthm}
\usepackage{amsfonts}
\usepackage{mathtools}
\usepackage{bbm}
\usepackage{pifont}
\usepackage{dsfont}
\usepackage{microtype}
\usepackage{booktabs}
\usepackage{subcaption}
\usepackage{graphicx}
\usepackage{algorithm}
\usepackage[noend]{algorithmic}
\usepackage{booktabs}
\usepackage{multirow}

\newcommand{\alglinelabel}{%
  \addtocounter{ALC@line}{-1}
  \refstepcounter{ALC@line}
  \label
}

\usepackage{thmtools}
\usepackage{thm-restate}

\usepackage[hidelinks]{hyperref}
\hypersetup{colorlinks=True, linkcolor=blue, citecolor=blue}

\usepackage[capitalize,noabbrev]{cleveref}
\creflabelformat{equation}{#2#1#3}

\newcommand{\calM}{\ensuremath{\mathcal{M}}}
\newcommand{\calN}{\ensuremath{\mathcal{N}}}

\newcommand{\calP}{\ensuremath{\mathcal{P}}}
\newcommand{\calQ}{\ensuremath{\mathcal{Q}}}

\newcommand{\calW}{\ensuremath{\mathcal{W}}}
\newcommand{\calX}{\ensuremath{\mathcal{X}}}
\newcommand{\calY}{\ensuremath{\mathcal{Y}}}

\newtheorem{lemma}{Lemma}[section]

\newtheorem{theorem}[lemma]{Theorem}

\newtheorem{definition}[lemma]{Definition}

\newcommand{\adassp}{\mathsf{AdaSSP}}
\newcommand{\alg}{\mathsf{TukeyEM}}

\newcommand{\dpsgd}{\mathsf{DPSGD}}
\newcommand{\E}[2]{\mathbb{E}_{#1}\left[ #2 \right]}
\newcommand{\eps}{\varepsilon}
\newcommand{\indic}[1]{\mathbbm{1}_{#1}}
\newcommand{\lap}[1]{\mathsf{Lap}\left(#1\right)}
\newcommand{\nondp}{\mathsf{NonDP}}
\renewcommand{\P}[2]{\mathbb{P}_{#1}\left[#2\right]}
\newcommand{\ptr}{\mathsf{PTRCheck}}
\newcommand{\safe}[1]{\mathtt{Safe}_{(#1)}}
\newcommand{\rtukeyem}{\mathsf{RestrictedTukeyEM}}
\newcommand{\samplepoint}{\mathsf{SamplePointWithDepth}}
\newcommand{\unsafe}[1]{\mathtt{Unsafe}_{(#1)}}

\newcommand{\vol}{\mathsf{vol}}

\usepackage[utf8]{inputenc}
\usepackage{floatpag}
\usepackage{float}
\usepackage{wrapfig}
\newfloat{algorithm}{t}{lop}
\usepackage{nicefrac}

\newcommand{\arxiv}[1]{}
\newcommand{\narxiv}[1]{#1}

\iclrfinalcopy
\begin{document}
\title{Easy Differentially Private \\ Linear Regression}

\author{Kareem Amin \And
Matthew Joseph \And
Mónica Ribero \And Sergei Vassilvitskii\thanks{\{kamin, mtjoseph, mribero, sergeiv\}@google.com. Part of this work done while Mónica was at UT Austin.}
}

\maketitle

\begin{abstract}
    Linear regression is a fundamental tool for statistical analysis. This has motivated the development of linear regression methods that also satisfy differential privacy and thus guarantee that the learned model reveals little about any one data point used to construct it. However, existing differentially private solutions assume that the end user can easily specify good data bounds and hyperparameters. Both present significant practical obstacles. In this paper, we study an algorithm which uses the exponential mechanism to select a model with high Tukey depth from a collection of non-private regression models. Given $n$ samples of $d$-dimensional data used to train $m$ models, we construct an efficient analogue using an approximate Tukey depth that runs in time $O(d^2n + dm\log(m))$. We find that this algorithm obtains strong empirical performance in the data-rich setting with no data bounds or hyperparameter selection required.
\end{abstract}

\section{Introduction}
\label{sec:intro}
Existing methods for differentially private linear regression include objective perturbation~\citep{KST12}, ordinary least squares (OLS) using  noisy sufficient statistics~\citep{DTTZ14, W18, S19}, and DP-SGD~\citep{ACGMMTZ16}. Carefully applied, these methods can obtain high utility in certain settings. However, each method also has its drawbacks. Objective perturbation and sufficient statistics require the user to provide bounds on the feature and label norms, and DP-SGD requires extensive hyperparameter tuning (of clipping norm, learning rate, batch size, and so on).

In practice, users of differentially private algorithms struggle to provide instance-specific inputs like feature and label norms without looking at the private data \citep{SSHSV22}. Unfortunately, looking at the private data also nullifies the desired differential privacy guarantee. Similarly, while recent work has advanced the state of the art of private hyperparameter tuning~\citep{LT19, PS22}, non-private hyperparameter tuning remains the most common and highest utility approach in practice. Even ignoring its (typically elided) privacy cost, this tuning adds significant time and implementation overhead. Both considerations present obstacles to differentially private linear regression in practice.

With these challenges in mind, the goal of this work is to provide an easy differentially private linear regression algorithm that works quickly and with no user input beyond the data itself. Here, ``ease'' refers to the experience of end users. The algorithm we propose requires care to construct and implement, but it only requires an end user to specify their dataset and desired level of privacy. We also emphasize that ease of use, while nice to have, is not itself the primary goal. Ease of use affects both privacy and utility, as an algorithm that is difficult to use will sacrifice one or both when data bounds and hyperparameters are imperfectly set. 

\subsection{Contributions}
Our algorithm generalizes previous work by~\citet{AMSSV22}, which proposes a differentially private variant of the Theil-Sen estimator for one-dimensional linear regression~\citep{T92}. The core idea is to partition the data into $m$ subsets, non-privately estimate a regression model on each, and then apply the exponential mechanism with some notion of depth to privately estimate a high-depth model from a restricted domain that the end user specifies. In the simple one-dimensional case~\citep{AMSSV22} each model is a slope, the natural notion of high depth is the median, and the user provides an interval for candidate slopes.

We generalize this in two ways to obtain our algorithm, $\alg$. The first step is to replace the median with a multidimensional analogue based on Tukey depth. Second, we adapt a technique based on propose-test-release (PTR), originally introduced by~\citet{BGSUZ21} for private estimation of unbounded Gaussians, to construct an algorithm which does not require bounds on the domain for the overall exponential mechanism. We find that a version of $\alg$ using an approximate and efficiently computable notion of Tukey depth achieves empirical performance competitive with (and often exceeding) that of \emph{non-privately tuned} baseline private linear regression algorithms, across several synthetic and real datasets. We highlight that the approximation only affects utility and efficiency; $\alg$ is still differentially private. Given an instance where $\alg$ constructs $m$ models from $n$ samples of $d$-dimensional data, the main guarantee for our algorithm is the following:
\begin{theorem}
\label{thm:main_approx}
    $\alg$~is $(\eps, \delta)$-DP and takes time $O(d^2n + dm\log(m))$.
\end{theorem}

\arxiv{It's instructive to compare the runing time to non-private OLS. The first term captures the running time of computing $m$ OLS models on $d$ dimensional datasets, each of size $\nicefrac{n}{m}$. The second term bounds the running time of privately sampling a high Tukey depth model. For the reasonable range of $m$, $1 \leq m \leq \nicefrac{n}{d}$, the second term is subsumed by the first, unless $n \gg 2^{d^2}$.}

Two caveats apply. First, our use of PTR comes at the cost of an approximate $(\eps, \delta)$-DP guarantee as well as a failure probability: depending on the dataset, it is possible that the PTR step fails, and no regression model is output. Second, the algorithm technically has one hyperparameter, the number $m$ of models trained. Our mitigation of both issues is empirical. Across several datasets, we observe that a simple heuristic about the relationship between the number of samples $n$ and the number of features $d$, derived from synthetic experiments, typically suffices to ensure that the PTR step passes and specifies a high-utility choice of $m$. For the bulk of our experiments, the required relationship is on the order of $n \gtrsim 1000 \cdot d$. We emphasize that this heuristic is based only on the data dimensions $n$ and $d$ and does not require further knowledge of the data itself.\arxiv{\footnote{For simplicity, this paper uses swap privacy, where $n$ is not private data. However, since the requirements on $n$ are coarse, a conservative noisy lower bound on $n$ suffices in the add-remove model.}}

\arxiv{\paragraph{Organization.} Basic preliminaries appear in Section~\ref{sec:prelims}. A full exposition of our private algorithm using approximate Tukey depth appears in Section~\ref{sec:algorithm}, with experiments in Section~\ref{sec:experiments}. We conclude with a discussion and some future directions in Section~\ref{sec:conclusion}.}

\subsection{Related Work}
Linear regression is a specific instance of the more general problem of convex optimization. Ignoring dependence on the parameter and input space diameter for brevity, DP-SGD~\citep{BST14} and objective perturbation~\citep{KST12} obtain the optimal $O(\sqrt{d}/\eps)$ error for empirical risk minimization. AdaOPS and AdaSSP also match this bound~\citep{W18}. Similar results are known for population loss~\citep{BFTT19}, and still stronger results using additional statistical assumptions on the data~\citep{CWZ20, VTJ22}. Recent work provides theoretical guarantees with no boundedness assumptions on the features or labels~\citep{MKFI22} but requires bounds on the data's covariance matrix to use an efficient subroutine for private Gaussian estimation and does not include an empirical evaluation. The main difference between these works and ours is empirical utility without data bounds and hyperparameter tuning.

Another relevant work is that of~\citet{LKO21}, which also composes a PTR step adapted from~\citet{BGSUZ21} with a call to a restricted exponential mechanism\arxiv{ for estimating parameters, with output domain restricted to parameters of low loss}. The main drawback of this work is that, as with the previous work~\citep{BGSUZ21}, neither the PTR step nor the restricted exponential mechanism step is efficient. This applies to other works that have applied Tukey depth to private estimation as well~\citep{BMNS19, KSS20, LKKO21, RC21}. The main difference between these works and ours is that our approach produces an efficient, implemented mechanism.

Finally, concurrent independent work by~\citet{C22} also studies the usage of Tukey depth, as well as the separate notion of regression depth, to privately select from a collection of non-private regression models. A few differences exist between their work and ours. First, they rely on additive noise scaled to smooth sensitivity to construct a private estimate of a high-depth point. Second, their methods are not computationally efficient beyond small $d$, and are only evaluated for $d \leq 2$. Third, their methods require the end user to specify bounds on the parameter space.
\section{Preliminaries}
\label{sec:prelims}
We start with the definition of differential privacy, using the ``add-remove'' variant.

\begin{definition}[\cite{DMNS06}]
Databases $D,D'$ from data domain $\calX$ are \emph{neighbors}, denoted $D \sim D'$, if they differ in the presence or absence of a single record. A randomized mechanism $\calM:\calX \to \calY$ is \emph{$(\eps, \delta)$-differentially private} (DP) if for all $D \sim D'\in \calX$ and any $S\subseteq \calY$
\begin{equation*}
    \P{\calM}{\calM(D) \in S} \leq e^{\eps}\P{\calM}{\calM(D') \in S} + \delta.
\end{equation*}
\arxiv{When $\delta=0$, $\calM$ is $\eps$-DP.}
\end{definition}

\narxiv{When $\delta=0$, $\calM$ is $\eps$-DP.} One general $\eps$-DP algorithm is the exponential mechanism.

\begin{definition}[\cite{MT07}]
\label{def:em}
    Given database $D$ and utility function $u:\calX \times \calY \to \mathbb{R}$ mapping $(\text{database}, \text{output})$ pairs to scores with sensitivity
    \[
        \Delta_u =\max_{D \sim D', y \in \calY} |u(D,y) -u(D',y)|,
    \]
    the \emph{exponential mechanism} selects item $y \in \calY$ with probability proportional to $\exp{(\frac{\epsilon u(D,y)}{2\Delta_u})}$. We say the utility function $u$ is \emph{monotonic} if, for $D_1 \subset D_2$, for any $y$, $u(D_1, y) \leq u(D_2, y)$. Given monotonic $u$, the 2 inside the exponent denominator can be dropped.
\end{definition}

\begin{lemma}[\cite{MT07}]
    The exponential mechanism is $\epsilon$-DP.
\end{lemma}

Finally, we define Tukey depth.

\begin{definition}[\cite{T75}]
\label{def:tukey}
    A \emph{halfspace} $h_v$ is defined by a vector $v \in \mathbb{R}^d$, $h_v = \{y \in \mathbb{R}^d \mid \langle v, y \rangle \geq 0\}$. Let $D \subset \mathbb{R}^d$ be a collection of $n$ points. The \emph{Tukey depth} $T_D(y)$ of a point $y \in  \mathbb{R}^d$ with respect to $D$ is the minimum number of points in $D$ in any halfspace containing $y$,
    \[
        T_D(y) =  \min_{h_v \mid y \in h_v} \sum_{x \in D} \indic{x \in h_v}.
    \]
\end{definition}

Note that for a collection of $n$ points, the maximum possible Tukey depth is $\nicefrac{n}{2}$. We will prove a theoretical utility result for a version of our algorithm that uses exact Tukey depth. However, Tukey depth is NP-hard to compute for arbitrary $d$~\citep{J78}, so our experiments instead use a notion of approximate Tukey depth that can be computed efficiently. The approximate notion of Tukey depth only takes a minimum over the $2d$ halfspaces corresponding to the canonical basis.

\begin{definition}
\label{def:tukey_approx}
Let $E = \{e_1, ..., e_d\}$ be the canonical basis for $\mathbb{R}^d$ and let $D \subset \mathbb{R}^d$. The \emph{approximate Tukey depth} of a point $y \in \mathbb{R}^d$ with respect to $D$, denoted $\tilde T_D(y)$, is the minimum number of points in $D$ in any of the $2d$ halfspaces determined by $E$ containing $y$,
\[
    \tilde T_D(y) = \min_{e_j \mid e_j \in \pm E, y \in h_{y_i \cdot e_j}} \sum_{x \in D} \indic{x \in h_{y_i \cdot e_j}}.
\]
\end{definition}

Stated more plainly, approximate Tukey depth only evaluates depth with respect to the $d$ axis-aligned directions.\arxiv{\footnote{Another reasonable approach evaluates depth with respect to random directions. However, measuring and sampling from the resulting regions of different depth, which are intersections of random halfspaces, is significantly more complicated. With the canonical basis, regions of different depth are simply rectangles.}} A simple illustration of the difference between exact and approximate Tukey depth appears in \cref{fig:depth_example} \narxiv{in the Appendix's \cref{subsec:depth_illustration}}. For both exact and approximate Tukey depth, when $D$ is clear from context, we omit it for neatness.

\arxiv{\begin{figure}
    \centering
    \includegraphics[scale=0.5]{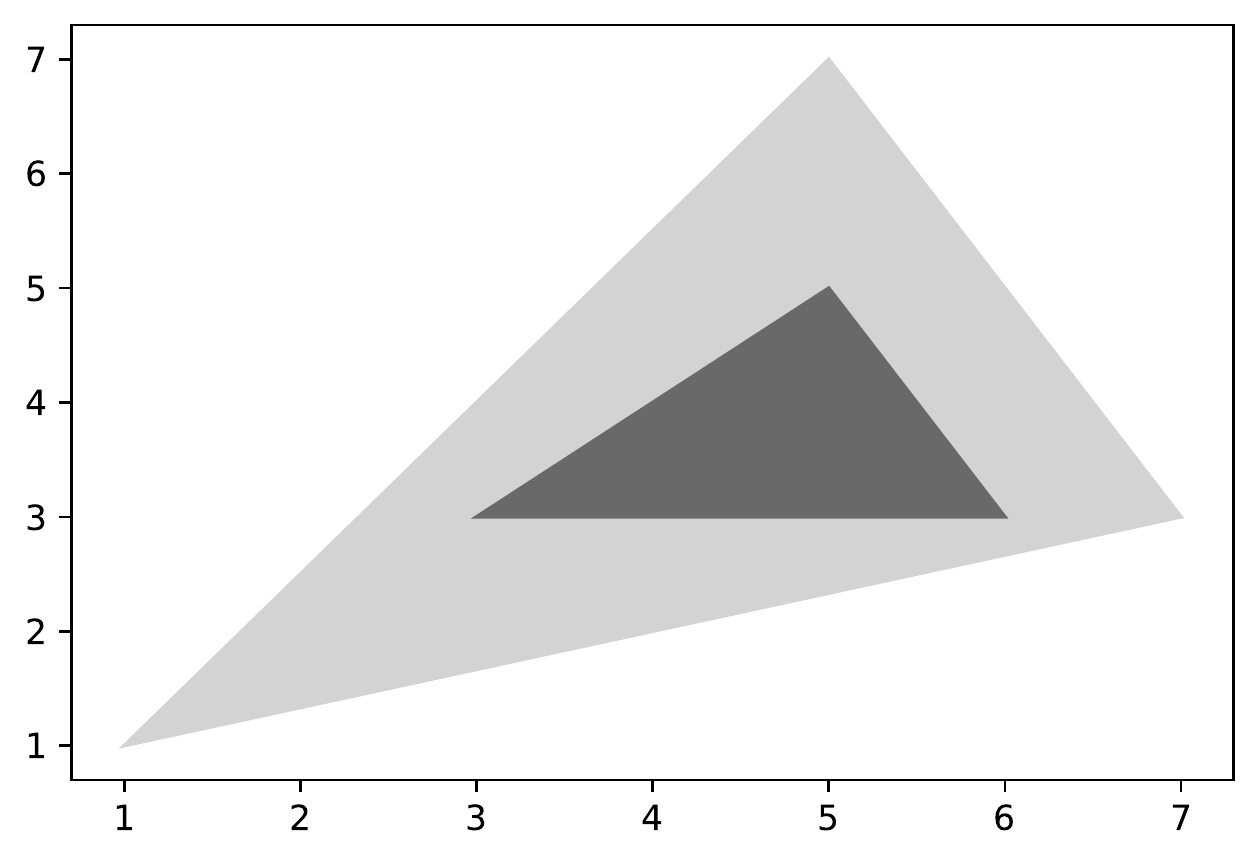}
    \includegraphics[scale=0.5]{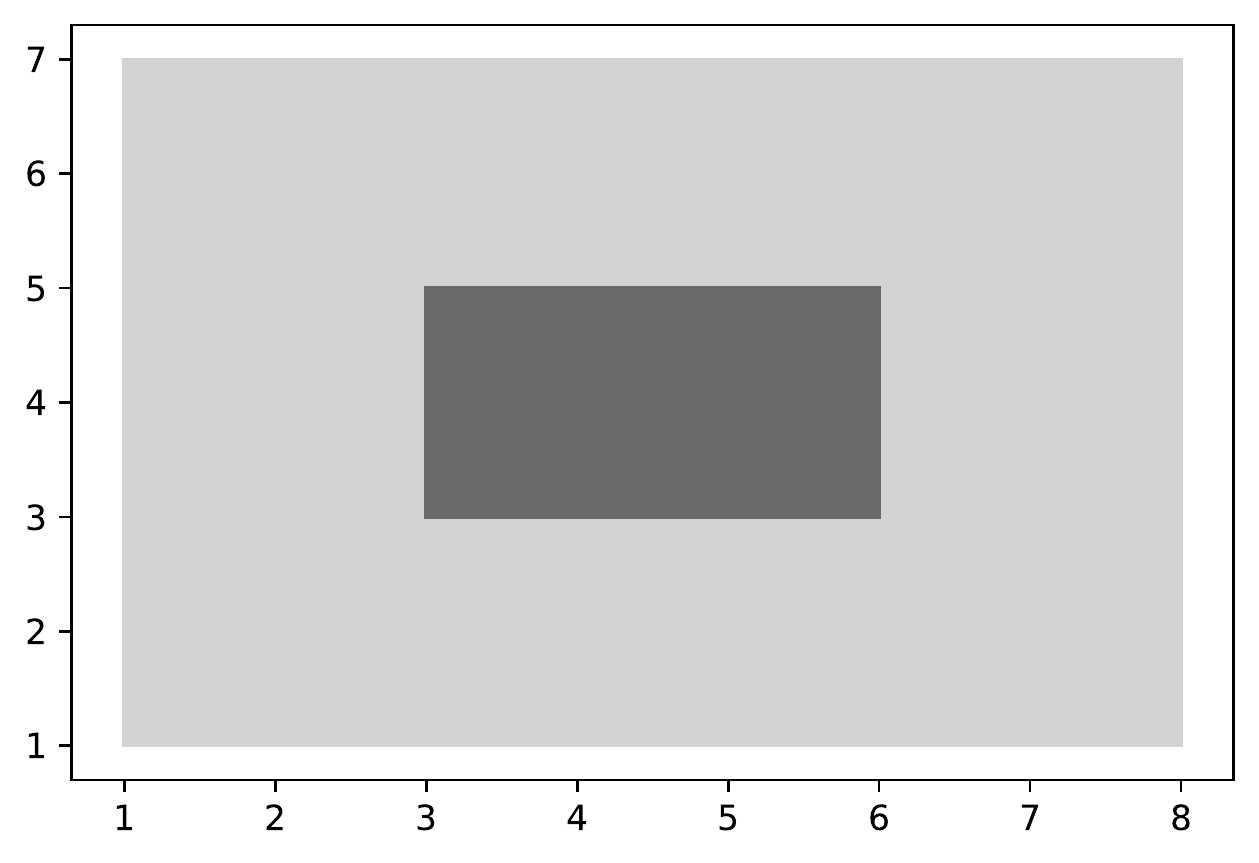}
    \caption{An illustrated comparison between exact (left) and approximate (right) Tukey depth. In both figures, the set of points is $\{(1, 1), (7, 3), (5, 7), (3, 3), (5, 5), (6, 3)\}$, the region of depth 0 is white, the region of depth 1 is light gray, and the region of depth 2 is dark gray. Note that for exact Tukey depth, the regions of different depths form a sequence of nested convex polygons; for approximate Tukey depth, they form a sequence of nested rectangles.}
    \label{fig:depth_example}
\end{figure}}

\section{Main Algorithm}
\label{sec:algorithm}
Our algorithm, $\alg$, consists of four steps:
\begin{enumerate}
    \item Randomly partition the dataset into $m$ subsets, non-privately compute the OLS estimator on each subset, and collect the $m$ estimators into set $\{\beta_i\}_{i=1}^m$.
    \item Compute the  volumes of regions of different approximate Tukey depths with respect to $\{\beta_i\}_{i=1}^m$.
    \item Run a propose-test-release (PTR) algorithm using these volumes. If it passes, set $B$ to be the region of $\mathbb{R}^d$ with approximate Tukey depth at least $\nicefrac{m}{4}$ in $\{\beta_i\}_{i=1}^m$ and proceed to the next step. If not, release $\bot$ (failure).
    \item If the previous step succeeds, apply the exponential mechanism, using approximate Tukey depth as the utility function, to privately select a point from $B$.
\end{enumerate}
A basic utility result for the version of $\alg$ using exact Tukey depth appears \arxiv{in the next subsection} \narxiv{below. \narxiv{The result is a direct application of work from \citet{BGSUZ21}, and the (short) proof appears in the Appendix's \cref{subsec:omitted_proofs}.}}. \arxiv{The remaining subsections elaborate on the details of our version using approximate Tukey depth, culminating in the full pseudocode in \cref{alg:main} and overall result, Theorem~\ref{thm:main_approx}.}

\arxiv{\subsection{Exact Utility Result}
\label{subsec:exact}
To build intuition, we first state a utility guarantee for the (inefficient) version of our algorithm that uses exact Tukey depth. The result is a direct application of a result from \citet{BGSUZ21} showing that a point of high Tukey depth with respect to samples from a Gaussian distribution $P$ approximates $\E{}{P}$ in Mahalanobis distance.}

\begin{theorem}[\cite{BGSUZ21}]
\label{thm:tukey-depth-convergence}
  Let $0<\alpha,\gamma<1$ and let $S = \{ \beta_1, ..., \beta_m\}$ be an i.i.d. sample from the multivariate normal distribution $\calN(\beta^*, \Sigma)$  with covariance $\Sigma \in \mathbb{R}^{d\times d}$ and mean $\E{}{\beta_i} = \beta^* \in \mathbb{R}^d$ . Given $\hat{\beta}\in \mathbb{R}^d$ with Tukey depth at least $p$ with respect to $S$, there exists a constant $c>0$ such that when $m\geq c \left(\frac{d+\log (1/\gamma)}{\alpha^2} \right)$ with probability $1-\gamma$,  $\|\hat{\beta}-\beta^*\|_{\Sigma}\leq \Phi^{-1}(1-p/m+\alpha)$, where $\Phi$ denotes the CDF of  of the standard univariate Gaussian.
\end{theorem}

\arxiv{\begin{proof}
This is a direct application of the results of~\citet{BGSUZ21}. They analyzed a notion of probabilistic (and normalized) Tukey depth over samples from a distribution:  $T_{\calN(\mu,\Sigma)}(y) := \min_{v}\P{X \sim \calN(\mu, \Sigma)}{\langle X, v \rangle \geq \langle y,v\rangle}$. Their Proposition 3.3 shows that $T_{\calN(\mu,\Sigma)}(y)$ can be characterized in terms of $\Phi$, the CDF of the standard one-dimensional Gaussian distribution. Specifically, they show $T_{\calN(\mu,\Sigma)}(y) = \Phi(-\|y-\mu \|_{\Sigma})$. From their Lemma 3.5, if $m\geq c \left(\frac{d+\log (1/\gamma)}{\alpha^2} \right)$, then with probability 
$1-\gamma$, $|p/m - T_{\calN(\beta^*,\Sigma)}(\hat{\beta})| \leq \alpha$. Thus
\begin{align*}
    -\alpha \leq&\ T_{\calN(\beta^*, \Sigma)}(\hat \beta) - p/m \\
    p/m - \alpha \leq&\ \Phi(-\|\hat \beta - \beta^*\|_\Sigma) \\
    p/m - \alpha \leq&\ 1 - \Phi(\|\hat \beta - \beta^*\|_\Sigma) \\
    \Phi(\|\hat \beta - \beta^*\|_\Sigma) \leq&\ 1 - p/m + \alpha \\
    \|\hat \beta - \beta^*\|_\Sigma \leq&\ \Phi^{-1}(1 - p/m + \alpha)
\end{align*}
where the third inequality used the symmetry of $\Phi$.
\end{proof}}

In practice, we observe that empirical distributions of models for real data often feature Gaussian-like concentration, fast tail decay, and symmetry.
\arxiv{\cref{fig:hist-housing} displays histograms for the coefficients of models learned by $\alg$ on one of our experiment datasets. Similar plots for the other datasets appear in \cref{subsec:model_distributions}.}
\narxiv{Plots of histograms for the the models learned by $\alg$ on experiment datasets appear in the Appendix's \cref{subsec:model_distributions}.}
Nonetheless, we emphasize that \cref{thm:tukey-depth-convergence} is a statement of sufficiency, not necessity. $\alg$ does not require any distributional assumption to be private, nor does non-Gaussianity preclude accurate estimation.   

\narxiv{The remaining subsections elaborate on the details of our version using approximate Tukey depth, culminating in the full pseudocode in \cref{alg:main} and overall result, Theorem~\ref{thm:main_approx}.}

\arxiv{\begin{figure}
    \centering
    \includegraphics[width = \textwidth]{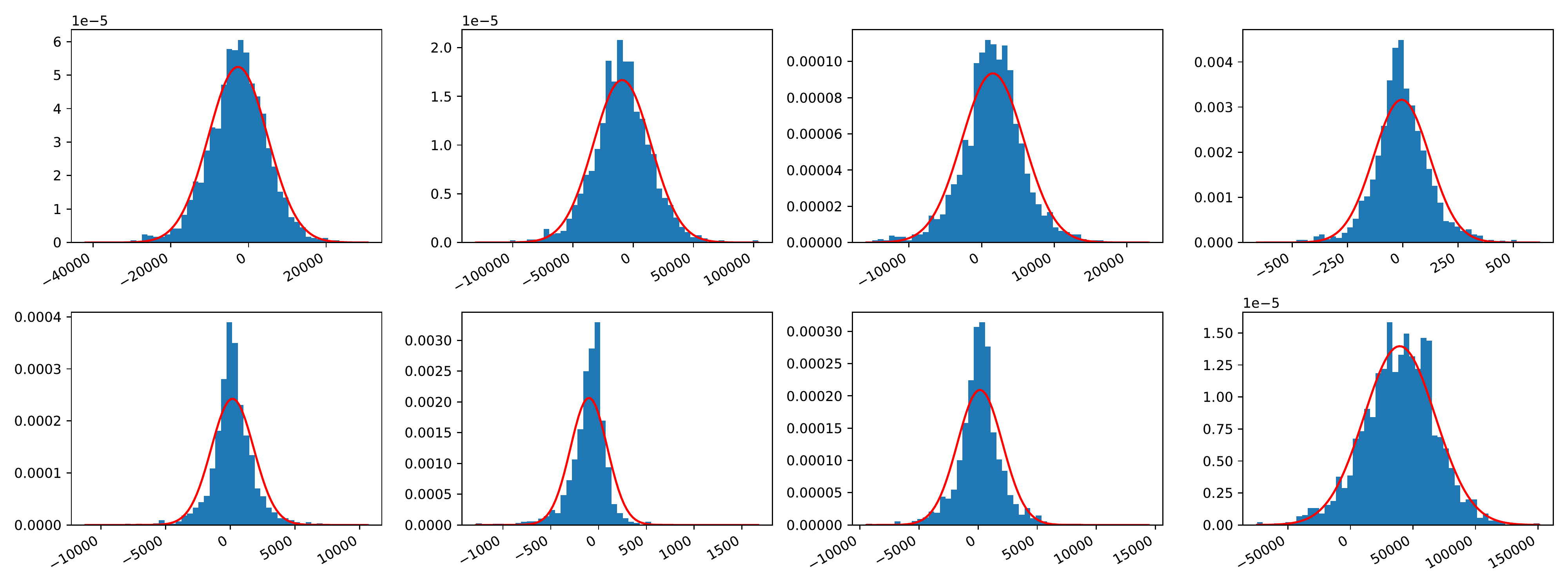}
    \caption{Histograms of models on the California dataset (see \cref{subsec:datasets}). Each plot corresponds to one of the eight regression coefficients. The red curve plots a Gaussian distribution whose mean and standard deviation match that of the underlying data.}
    \label{fig:hist-housing}
\end{figure}}

\subsection{Computing Volumes}
\label{subsec:volumes}
We start by describing how to compute volumes corresponding to different Tukey depths. As shown in the next subsection, these volumes will be necessary for the PTR subroutine.

\begin{definition}
\label{def:Vs}
    Given database $D$, define $V_{i, D} = \vol(\{y \mid y \in \mathbb{R}^d \text{ and } \tilde T_D(y) \geq i\})$, the volume of the region of points in $\mathbb{R}^d$ with approximate Tukey depth at least $i$ in $D$. When $D$ is clear from context, we write $V_i$ for brevity.
\end{definition}

Since our notion of approximate Tukey depth uses the canonical basis (Definition~\ref{def:tukey_approx}), it follows that $V_1, V_2, \ldots, V_{m/2}$\footnote{We assume $m$ is even for simplicity. The algorithm and its guarantees are essentially the same when $m$ is odd, and our implementation handles both cases.} form a sequence of nested (hyper)rectangles, as shown in \cref{fig:depth_example}. With this observation, computing a given $V_i$ is simple. For each axis, project the non-private models $\{\beta_i\}_{i=1}^m$ onto the axis and compute the distance between the two points of exact Tukey depth $i$ (from the ``left'' and ``right'') in the one-dimensional sorted array. This yields one side length for the hyperrectangle. Repeating this $d$ times in total and taking the product then yields the total volume of the hyperrectangle, as formalized next. \narxiv{The simple proof appears in the Appendix's \cref{subsec:omitted_proofs}.}
\begin{lemma}
\label{lem:compute_volumes}
    Lines~\ref{algln:compute_S} to~\ref{algln:volumes_end} of Algorithm~\ref{alg:main} compute $\{V_i\}_{i=1}^{m/2}$ in time $O(dm\log(m))$.
\end{lemma}
\arxiv{\begin{proof}
     By the definition of approximate Tukey depth, for arbitrary $y = (y_1, \ldots, y_d)$ of Tukey depth at least $i$, each of the $2d$ halfspaces $h_{y_1 \cdot e_1}, h_{y_1 \cdot -e_1}, \ldots, h_{y_d \cdot e_d}, h_{y_d \cdot -e_d}$ contains at least $i$ points from $D$, where $x \cdot y$ denotes multiplication of a scalar and vector. Fix some dimension $j \in [d]$.  Since $\min(|h_{y_j \cdot e_j} \cap D|, |h_{y_j \cdot -e_j} \cap D|) \geq i$, $y_j \in [S_{j, i}, S_{j, m-(i-1)}]$. Thus $V_{i, D} = \prod_{j=1}^d (S_{j, m-(i-1)} - S_{j, i})$. The computation of $S$ starting in Line~\ref{algln:compute_S} sorts $d$ arrays of length $m$ and so takes time $O(dm\log(m))$. Line~\ref{algln:volumes_start} iterates over $m/2$ depths and computes $d$ quantities, each in constant time, so its total time is $O(dm)$.
\end{proof}}

\subsection{Applying Propose-Test-Release}
\label{subsec:ptr}
The next step of $\alg$~employs PTR to restrict the output region eventually used by the exponential mechanism. We collect this process into a subroutine $\ptr{}$.

The overall strategy applies work done by~\citet{BGSUZ21}. Their algorithm privately checks if the given database has a large Hamming distance to any ``unsafe'' database and then, if this PTR check passes, runs an exponential mechanism restricted to a domain of high Tukey depth. Since a ``safe'' database is defined as one where the restricted exponential mechanism has a similar output distribution on any neighboring database, the overall algorithm is DP. As part of their utility analysis, they prove a lemma translating a volume condition on regions of different Tukey depths to a lower bound on the Hamming distance to an unsafe database (Lemma 3.8 \citep{BGSUZ21}). This enables them to argue that the PTR check typically passes if it receives enough Gaussian data, and the utility guarantee follows. However, their algorithm requires computing both exact Tukey depths of the samples and the current database's exact Hamming distance to unsafety. The given runtimes for both computations are exponential in the dimension $d$ (see their Section C.2~\citep{BGSUZ21}).

We rely on approximate Tukey depth (Definition~\ref{def:tukey_approx}) to resolve both issues. First, as the previous section demonstrated, computing the approximate Tukey depths of a collection of $m$ $d$-dimensional points only takes time $O(dm\log(m))$. Second, we adapt their lower bound to give a \emph{1-sensitive} lower bound on the Hamming distance between the current database and any unsafe database. This yields an efficient replacement for the exact Hamming distance calculation used by \citet{BGSUZ21}. 

The overall structure of $\ptr{}$ is therefore as follows: use the volume condition to compute a 1-sensitive lower bound on the given database's distance to unsafety; add noise to the lower bound and compare it to a threshold calibrated so that an unsafe dataset has probability $\leq \delta$ of passing; and if the check passes, run the exponential mechanism to pick a point of high approximate Tukey depth from the domain of points with moderately high approximate Tukey depth. Before proceeding to the details of the algorithm, we first define a few necessary terms.
\begin{definition}[Definition 2.1 \cite{BGSUZ21}]
    Two distributions $\calP, \calQ$ over domain $\calW$ are \emph{$(\eps, \delta)$-indistinguishable}, denoted $\calP \approx_{\eps, \delta} \calQ$, if for any measurable subset $W \subset \calW$,
    \[
        \P{w \sim \calP}{w \in W} \leq e^\eps\P{w \sim \calQ}{w \in W} + \delta \text{ and } \P{w \sim \calQ}{w \in W} \leq e^\eps\P{w \sim \calP}{w \in W} + \delta.
    \]
\end{definition}
Note that $(\eps,\delta)$-DP is equivalent to $(\eps, \delta)$-indistinguishability between output distributions on arbitrary neighboring databases. \arxiv{Indistinguishability will define collections of ``safe'' and ``unsafe'' databases.} Given database $D$, let $A$ denote the exponential mechanism with utility function $\tilde T_D$ (see Definition~\ref{def:tukey_approx}). Given nonnegative integer $t$, let $A_t$ denote the same mechanism that assigns score $-\infty$ to any point with score $< t$, i.e., only samples from points of score $\geq t$. We will say a database is ``safe'' if $A_t$ is indistinguishable between neighbors.
\begin{definition}[Definition 3.1~\cite{BGSUZ21}]
\label{def:safety}
    Database $D$ is $(\eps, \delta, t)$-safe if for all neighboring $D' \sim D$, we have $A_t(D) \approx_{\eps, \delta} A_t(D')$. Let $\safe{\eps, \delta, t}$ be the set of safe databases, and let $\unsafe{\eps, \delta, t}$ be its complement.
\end{definition}

\arxiv{We can now state the main lemma from~\citet{BGSUZ21}, which translates a relationship between the volumes of regions of different approximate Tukey depths into a lower bound on the Hamming distance between the given database and any unsafe one. 
\begin{lemma}[Lemma 3.8~\citep{BGSUZ21}]
\label{lem:ptr}
    For any $k \geq 0$, if there exists a $g > 0$ such that $\frac{V_{t-k-1, D}}{V_{t+k+g+1, D}} \cdot e^{-\eps g/ 2} \leq \delta$, then for every database $z$ in $\unsafe{\eps, 4e^\eps\delta, t}$, $d_H(D, z) > k$, where $d_H$ denotes Hamming distance.
\end{lemma}
Simplifying slightly, the lemma says that if $V_{t-k, D}$ is not substantially larger than $V_{t+k+g, D}$, then making $D$ unsafe requires changing more than $k$ points. This relies on the usage of an exponential mechanism whose utility function has sensitivity 1. 

One more step is necessary to complete $\ptr$. Lemma~\ref{lem:ptr} provides a way to compute a lower bound on the distance between the current database and any unsafe database, but it does not guarantee that this lower bound is 1-sensitive. We use this result to construct a 1-sensitive lower bound in Lemma~\ref{lem:ptr_2}.}

\narxiv{We now state the main result of this section, \cref{lem:ptr_2}. Briefly, it modifies Lemma 3.8 from \citet{BGSUZ21} to construct a 1-sensitive lower bound on distance to unsafety.}

\begin{lemma}
\label{lem:ptr_2}
    Define $M(D)$ to be a mechanism that receives as input database $D$ and computes the largest $k \in \{0, \ldots, t-1\}$ such that there exists $g > 0$ where, for volumes $V$ defined using a monotonic utility function,
    \[
        \frac{V_{t-k-1, D}}{V_{t+k+g+1,D}} \cdot e^{-\eps g/2} \leq \delta
    \]
    or outputs $-1$ if the inequality does not hold for any such $k$. Then for arbitrary $D$
    \begin{enumerate}
        \item $M$ is 1-sensitive, and
        \item for all $z \in \unsafe{\eps, 4e^\eps\delta, t}$, $d_H(D, z) > M(D)$.
    \end{enumerate}
\end{lemma}

\narxiv{The proof of \cref{lem:ptr_2} appears in the Appendix's \cref{subsec:omitted_proofs}.} Our implementation of the algorithm described by \cref{lem:ptr_2} randomly perturbs the models with a small amount of noise to avoid having regions with 0 volume. We note that this does not affect the overall privacy guarantee.

$\ptr$ therefore runs the mechanism defined by Lemma~\ref{lem:ptr_2}, add Laplace noise \arxiv{scaled to sensitivity 1} to the result, and proceeds to the restricted exponential mechanism if the noisy statistic crosses a threshold.\arxiv{ When $\ptr$ is called with parameter $\delta$, the threshold is set so that the probability of an unsafe dataset producing a sufficiently large noisy statistic is upper bounded by $\delta$.} Pseudocode appears in Algorithm~\ref{alg:ptr}, and we now state its guarantee \narxiv{as proved in the Appendix's \cref{subsec:omitted_proofs}}.

\begin{lemma}
\label{lem:ptr_dp}
    Given the depth volumes $V$ computed in Lines~\ref{algln:volumes_start} to~\ref{algln:volumes_end} of Algorithm~\ref{alg:main}, $\ptr(V, \eps, \delta)$ is $\eps$-DP and takes time $O(m\log(m))$.
\end{lemma}
\arxiv{\begin{proof}
    The privacy guarantee follows from the 1-sensitivity of computing $k$ (\cref{lem:ptr_2}). For the runtime guarantee, we perform a binary search over the distance lower bound $k$ starting from the maximum possible $\nicefrac{m}{4}$ and, for each $k$, an exhaustive search over $g$. Note that if some $k, g$ pair satisfies the inequality in Lemma~\ref{lem:ptr}, there exists some $g'$ for every $k' < k$ that satisfies it as well. Thus since both have range $\leq \nicefrac{m}{4}$, the total time is $O(m\log(m))$.
\end{proof}}

\begin{algorithm}
    \begin{algorithmic}[1]
       \STATE {\bfseries Input:} Tukey depth region volumes $V$, privacy parameters $\eps$ and $\delta$ \alglinelabel{algln:ptr_input}
       \STATE Use Lemma~\ref{lem:ptr_2} with $t = \frac{|V|}{2}$ and $\frac{\delta}{8e^\eps}$ to compute lower bound $k$ for distance to unsafe database
       \alglinelabel{algln:ptr_lower_bound}
       \IF{$k + \lap{1/\eps} \geq \frac{\log(1/2\delta)}{\eps}$}
       \alglinelabel{algln:ptr_check}
        \STATE Return True
       \ELSE
        \STATE Return False
       \ENDIF
    \end{algorithmic}
    \caption{$\ptr$}
    \label{alg:ptr}
\end{algorithm}

\subsection{Sampling}
\label{subsec:sample}
If $\ptr$ passes, $\alg$~then calls the exponential mechanism restricted to points of approximate Tukey depth at least $t = \nicefrac{m}{4}$, a subroutine denoted $\rtukeyem$ (Line~\ref{algln:rtukeyem} in Algorithm~\ref{alg:main}). Note that the passage of PTR ensures that with probability at least $1-\delta$, running $\rtukeyem$ is $(\eps, \delta)$-DP. We use a common two step process for sampling from an exponential mechanism over a continuous space: \narxiv{1) sample a depth using the exponential mechanism, then 2) return a uniform sample from the region corresponding to the sampled depth.}
\arxiv{\begin{enumerate}
    \item Sample a depth using the exponential mechanism.
    \item Return a uniform sample from the region corresponding to the sampled depth.
\end{enumerate}}

\subsubsection{Sampling a Depth}
\label{subsec:sample_depth}
We first define a slight modification $W$ of the volumes $V$ introduced earlier.
\begin{definition}
\label{def:Ws}
    Given database $D$, define $W_{i, D} = \vol(\{y \mid y \in \mathbb{R}^d \text{ and } \tilde T_D(y) = i\})$, the volume of the region of points in $\mathbb{R}^d$ with approximate Tukey depth exactly $i$ in $D$.
\end{definition}
To execute the first step of sampling, for $i \in \{m/4, m/4 + 1, \ldots, m/2\}$, $W_{i, D} = V_{i, D} - V_{i+1, D}$, so we can compute $\{W_{i, D}\}_{i=m/4}^{m/2}$ from the $V$ computed earlier in time $O(m)$. The restricted exponential mechanism then selects approximate Tukey depth $i \in \{m/4, m/4 + 1, \ldots, m/2\}$ with probability
\[
    \P{}{i} \propto W_{i, D} \cdot \exp(\eps \cdot i).
\]
Note that this expression drops the 2 in the standard exponential mechanism because approximate Tukey depth is monotonic; see Appendix \cref{subsec:monotone} for details. For numerical stability, racing sampling~\cite{MG21} can sample from this distribution using logarithmic quantities.

\subsubsection{Uniformly Sampling From a Region}
\label{subsec:sample_region}
Having sampled a depth $\hat i$, it remains to return a uniform random point of approximate Tukey depth $\hat i$. By construction, $W_{\hat{i}, D}$ is the volume of the set of points  $y = (y_1,...,y_d)$ such that the depth along every dimension $j$ is at least $\hat{i}$, and the depth along at least one dimension $j'$ is exactly $\hat{i}$. The result is straightforward when $d=1$: draw a uniform sample from the union of the two intervals of points of depth exactly $\hat i$ (depth from the ``left'' and ``right'').

For $d > 1$, the basic idea of the sampling process is to partition the overall volume into disjoint subsets, compute each subset volume, sample a subset according to its proportion in the whole volume, and then sample uniformly from that subset. Our partition will split the overall region of depth exactly $i$ according to the first dimension with dimension-specific depth exactly $i$. Since any point in the overall region has at least one such dimension, this produces a valid partition, and we will see that computing the volumes of these partitions is straightforward using the $S$ computed earlier. Finally, the last sampling step will be easy because the final subset will simply be a pair of (hyper)rectangles. \narxiv{Since space is constrained and the details are relatively straightforward from the sketch above, full specification and proofs for this process $\samplepoint(S, i)$ appear in \cref{subsec:sample_region_details}. For immediate purposes, it suffices to record the following guarantee:
\begin{lemma}
\label{lem:samplepoint}
    $\samplepoint(S, i)$ returns a uniform random sample from the region of points with approximate Tukey depth $i$ in $S$ in time $O(d)$.
\end{lemma}}

\subsection{Overall Algorithm}
\label{subsec:overall}
We now have all of the necessary material for the main result, \cref{thm:main_approx}, restated below. The proof essentially collects the results so far into a concise summary.

\begin{algorithm}
    \begin{algorithmic}[1]
        \STATE {\bfseries Input:} Features matrix $X \in \mathbb{R}^{n \times d}$, label vector $y \in \mathbb{R}^n$, number of models $m$, privacy parameters $\eps$ and $\delta$ 
        \alglinelabel{algln:input}
        \STATE Evenly and randomly partition $X$ and $y$ into subsets $\{(X_i, y_i)\}_{i=1}^m$
        \alglinelabel{algln:partition}
        \FOR{$i = 1, \ldots, m$}
        \alglinelabel{algln:ols}
            \STATE Compute OLS estimator $\beta_i \gets (X_i^TX_i)^{-1}X_i^Ty_i$
        \ENDFOR
        \FOR{dimension $j \in [d]$}
        \alglinelabel{algln:compute_S}
            \STATE $\{\beta_{i, j}\}_{i=1}^m \gets$ projection of $\{\beta_i\}_{i=1}^m$ onto dimension $j$
            \STATE $(S_{j, 1}, \ldots, S_{j, m}) \gets \{\beta_{i, j}\}_{i=1}^m$ sorted in nondecreasing order 
        \ENDFOR
        \STATE Collect projected estimators into $S \in \mathbb{R}^{d \times m}$, where each row is nondecreasing
        \alglinelabel{algln:S}
        \FOR{$i \in [m/2]$}
        \alglinelabel{algln:volumes_start}
            \STATE Compute volume of region of depth $\geq i$, $V_i \gets \prod_{j=1}^d (S_{j, m-(i-1)} - S_{j, i})$
        \ENDFOR
        \alglinelabel{algln:volumes_end}
        \IF{$\ptr(V, \eps / 2, \delta)$}
        \alglinelabel{algln:ptr}
            \STATE $\hat \beta \gets \rtukeyem(V, S, m/4, \eps/2)$ \alglinelabel{algln:rtukeyem}
            \STATE Return $\hat \beta$
        \ELSE
            \STATE Return $\bot$
        \ENDIF
    \end{algorithmic}
    \caption{$\alg$}
    \label{alg:main}
\end{algorithm}

\begin{theorem}
    $\alg$, given in \cref{alg:main}, is $(\eps, \delta)$-DP and takes time $O\left(d^2n + dm\log(m)\right)$.
\end{theorem}
\begin{proof}
    Line \ref{algln:ptr} of the $\alg$ pseudocode in Algorithm~\ref{alg:main} calls the check with privacy parameters $\eps/2$ and $\delta/[8e^\eps]$. By the sensitivity guarantee of \cref{lem:ptr_2}, the check itself is $\eps/2$-DP. By the safety guarantee of \cref{lem:ptr_2} and our choice of threshold, if it passes, with probability at least $1-\delta/2$, the given database lies in $\safe{\eps/2, \delta/2, t}$. A passing check therefore ensures that the sampling step in Line~\ref{algln:rtukeyem} is $(\eps/2, \delta)$-DP. By composition, the overall privacy guarantee is $(\eps, \delta)$-DP. Turning to runtime, the $m$ OLS computations inside Line~\ref{algln:ols} each multiply $d \times \tfrac{n}{m}$ and $\tfrac{n}{m} \times d$ matrices, for $O(d^2n)$ time overall. From \cref{lem:compute_volumes}, Lines~\ref{algln:compute_S} to~\ref{algln:volumes_end} take time $O(dm\log(m))$. \cref{lem:ptr_dp} gives the $O(m\log(m))$ time for Line~\ref{algln:ptr}, and \cref{lem:samplepoint} gives the $O(d)$ time for Line~\ref{algln:rtukeyem}.
\end{proof}
\section{Experiments}
\label{sec:experiments}
\arxiv{We now turn to an empirical evaluation of $\alg$. A description of baseline comparison algorithms appears in Section~\ref{subsec:baselines}, and details for the datasets used appear in Section~\ref{subsec:datasets}. Section~\ref{subsec:setting_m} describes experiments on synthetic data to investigate the role of the number of models $m$ in $\alg$'s performance. We then present and discuss experiments evaluating accuracy (\cref{subsec:accuracy}) and runtime (\cref{subsec:time}). All experiment code can be found on Github~\citep{G22}.}

\subsection{Baselines}
\label{subsec:baselines}
\begin{enumerate}
    \item $\nondp$ computes the standard non-private OLS estimator $\beta^* = (X^TX)^{-1}X^Ty$.
    \item $\adassp$~\citep{W18} computes a DP OLS estimator based on noisy versions of $X^TX$ and $X^Ty$. This requires the end user to supply bounds on both $\|X\|_2$ and $\|y\|_2$. Our implementation uses these values non-privately for each dataset. The implementation is therefore not private and represents an artificially strong version of $\adassp$. As specified by~\citet{W18}, $\adassp$ (privately) selects a ridge parameter and runs ridge regression\arxiv{ using the noisy sufficient statistics and the selected parameter}.
    \item $\dpsgd$~\citep{ACGMMTZ16} uses DP-SGD, as implemented in TensorFlow Privacy and Keras~\citep{C15}, to optimize mean squared error using a single linear layer. The layer's weights are regression coefficients. A discussion of hyperparameter selection appears in \cref{subsec:accuracy}. As we will see, appropriate choices of these hyperparameters is both dataset-specific and crucial to $\dpsgd$'s performance. Since we allow $\dpsgd$ to tune these non-privately for each dataset, our implementation of $\dpsgd$ is also artificially strong.
\end{enumerate}
\narxiv{All experiment code can be found on Github~\citep{G22}.}

\subsection{Datasets}
\label{subsec:datasets}
We evaluate all four algorithms on the following datasets. The first dataset is synthetic, and the rest are real. The datasets are intentionally selected to be relatively easy use cases for linear regression, as reflected by the consistent high $R^2$ for $\nondp$.\footnote{$R^2$\arxiv{, also known as the coefficient of determination,} measures the variation in labels accounted for by the features. $R^2=1$ is perfect, $R^2=0$ is the trivial baseline achieved by simply predicting the average label, and $R^2 < 0$ is worse than the trivial baseline.} However, we emphasize that, beyond the constraints on $d$ and $n$ suggested by Section~\ref{subsec:setting_m}, they have not been selected to favor $\alg$: all feature selection occurred before running any of the algorithms, and we include all datasets evaluated where $\nondp$ achieved a positive $R^2$. A complete description of the datasets appears both in the public code and the Appendix's \cref{subsec:features}. For each dataset, we additionally add an intercept feature.

\begin{enumerate}
    \item \textbf{Synthetic} ($d=11$, $n=22{,}000$, \cite{PVGMT+11}). This dataset uses \texttt{sklearn.make\_regression} and $N(0, \sigma^2)$ label noise with $\sigma = 10$.
    \item \textbf{California} ($d=9$, $n= 20{,}433$, \cite{N17}) predicting house price.
    \item \textbf{Diamonds} ($d=10$, $n = 53{,}940$, \cite{A17}), predicting diamond price. 
    \item \textbf{Traffic} ($d=3$, $n = 7{,}909$, \cite{N13}), predicting number of passenger vehicles.
    \item \textbf{NBA} ($d=6$, $n = 21{,}613$, \cite{L22}), predicting home team score.
    \item \textbf{Beijing} ($d=25$, $n = 159{,}375$, \cite{R18}), predicting house price.
    \item \textbf{Garbage} ($d=8$, $n = 18{,}810$, \cite{N22}), predicting tons of garbage collected.
    \item \textbf{MLB} ($d = 11$, $n = 140{,}657$, \cite{S18}), predicting home team score.
\end{enumerate}

\narxiv{}

\subsection{Choosing the Number of Models}
\label{subsec:setting_m}
Before turning to the results of this comparison, recall from Section~\ref{sec:algorithm} that $\alg$ privately aggregates $m$ non-private OLS models. \arxiv{We briefly discuss the selection of this parameter $m$. }\narxiv{If $m$ is too low, $\ptr$ will probably fail; if $m$ is too high, and each model is trained on only a small number of points, even a non-private aggregation of inaccurate models will be an inaccurate model as well.}
\arxiv{An appropriate choice of $m$ balances two desiderata: accuracy of the models and likelihood of passing $\ptr$. When $m$ is low, each model is trained on many data points and is likely to have high accuracy. However, lower $m$ increases the risk of failing $\ptr$, as $\ptr$ works best when the space of models is ``dense.'' At the other extreme, we could set $m$ to be as large as possible, $m = \lfloor n/d \rfloor$. However, this raises the possibility that even a non-private aggregation of many inaccurate models will be an inaccurate model as well. }

Experiments on synthetic data support this intuition. In \arxiv{Figure~\ref{fig:distance_vs_num_models}}\narxiv{the left plot in \cref{fig:heuristic_and_time}}, each solid line represents synthetic data with a different number of features, generated by the same process as the Synthetic dataset described in the previous section. We vary the number of models $m$ on the $x$-axis and plot the distance computed by Lemma~\ref{lem:ptr_2}. As $d$ grows, the number of models required to pass the $\ptr$ threshold, demarcated by the dashed horizontal line, grows as well.

\arxiv{\begin{figure*}[h]
    \centering
    \includegraphics[scale=0.7]{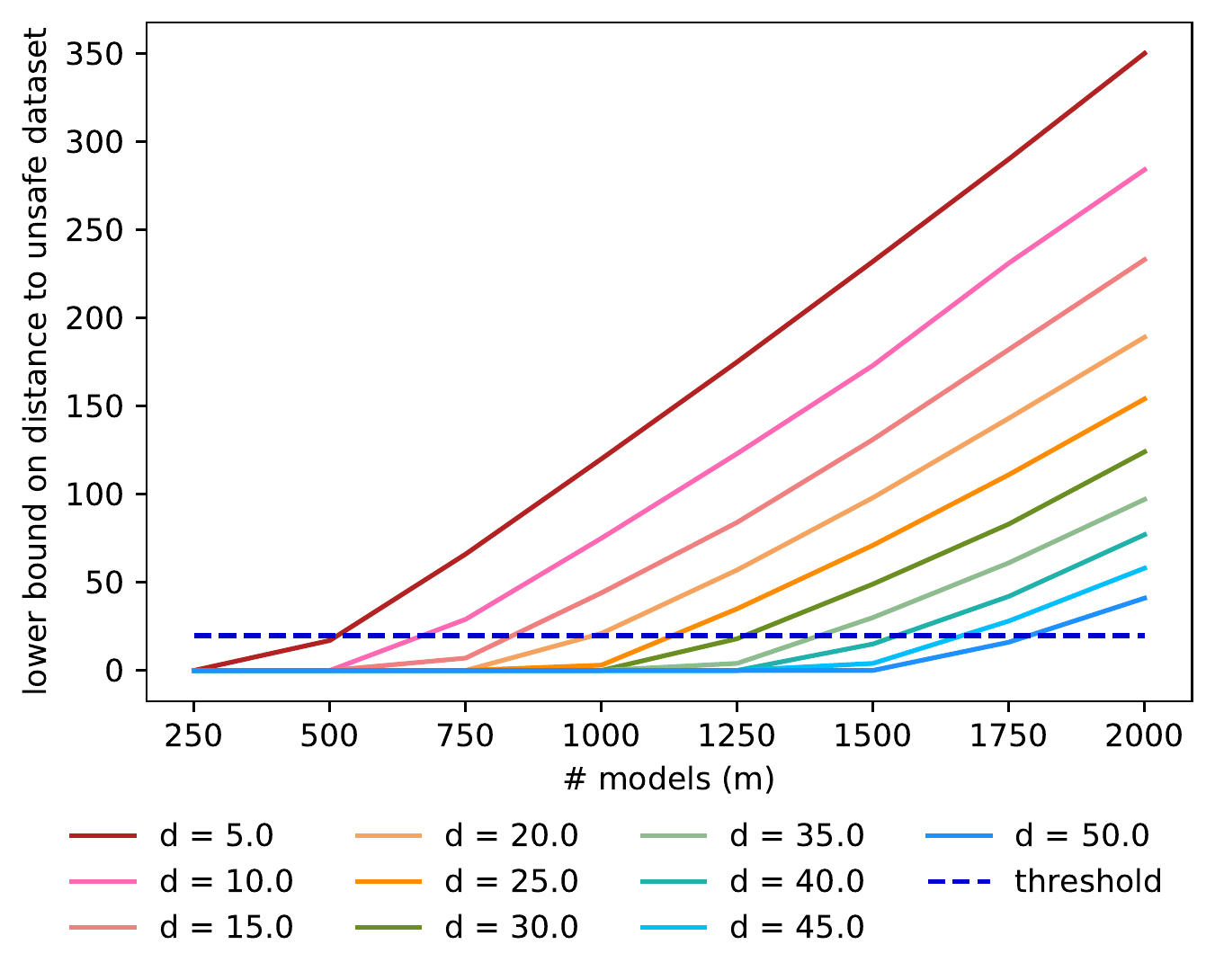}
    \caption{A plot of the lower bound on the Hamming distance to an unsafe dataset, computed from Lemma~\ref{lem:ptr}, as the number of models $m$ ranges over $250, 500, \ldots, 2000$, and the feature dimension $d$ ranges over $5, 10, \ldots, 50$ for synthetic data. The threshold for PTR to pass is the dashed horizontal line. As might be expected, larger $d$ requires larger $m$ to pass this threshold.}
    \label{fig:distance_vs_num_models}
\end{figure*}}

\arxiv{One caveat for the plot in Figure~\ref{fig:distance_vs_num_models} is that the Synthetic regression problem is unusually easy. As shown in Section~\ref{subsec:accuracy}, all DP algorithms obtain near-perfect $R^2$; as demonstrated by the other datasets, this is atypical. It nonetheless suggests that a minimum of 500 models is necessary for the PTR step to pass even for $d \leq 5$.}

To select the value of $R^2$ used for $\alg$, we ran it on each dataset using $m = 250, 500, \ldots, 2000$ and selected the smallest $m$ where all PTR checks passed. We give additional details in \cref{subsec:accuracy_extended} but note here that the resulting choices closely track those given by \arxiv{\cref{fig:distance_vs_num_models}}\narxiv{\cref{fig:heuristic_and_time}}. Furthermore, across many datasets, simply selecting $m=1000$ typically produces nearly optimal $R^2$, with several datasets exhibiting little dependence on the exact choice of $m$.

\subsection{Accuracy Comparison}
\label{subsec:accuracy}
Our main experiments compare the four methods at $(\ln(3), 10^{-5})$-DP. A concise summary of the experiment results appears in \cref{fig:results}. For every method other than $\nondp$ (which is deterministic), we report the median $R^2$ values across the trials. For each dataset, the methods with interquartile ranges overlapping that of the method with the highest median $R^2$ are bolded. Extended plots recording $R^2$ for various $m$ appear in \cref{subsec:accuracy_extended}. All datasets use 10 trials, except for California and Diamonds, which use 50.

\begin{figure}[h]
\begin{center}
\begin{tabular}{ |c|c|c|c|c|c| } 
 \hline
 Dataset & $\nondp$ & $\adassp$ & $\alg$ & $\dpsgd$ (tuned) & $\dpsgd$ (90\% tuned)\\ 
 \hline \hline
 Synthetic & 0.997 & 0.991 & $\mathbf{0.997}$ & $\mathbf{0.997}$ & $\mathbf{0.997}$ \\ 
 \hline
 California & 0.637 & -1.285 & $\mathbf{0.099}$ & $\mathbf{0.085}$ & -1.03 \\ 
 \hline
 Diamonds & 0.907 & 0.216 & 0.307 & $\mathbf{0.828}$ & 0.371 \\
 \hline
 Traffic & 0.966 & 0.944 & $\mathbf{0.965}$ & 0.938 & 0.765 \\
 \hline
 NBA & 0.621 & 0.018 & $\mathbf{0.618}$ & 0.531 & 0.344 \\
 \hline
 Beijing & 0.702 & 0.209 & $\mathbf{0.698}$ & 0.475 & 0.302 \\
 \hline
 Garbage & 0.542 & 0.119 & $\mathbf{0.534}$ & 0.215 & 0.152 \\
 \hline
 MLB &  0.722 & 0.519 & $\mathbf{0.721}$ & 0.718 & 0.712 \\
 \hline
\end{tabular}
\end{center}
\caption{For each dataset, the DP methods with interquartile ranges overlapping that of the DP method with the highest median $R^2$ are bolded.} 
    \label{fig:results}
\end{figure}

A few takeaways are immediate. First, on most datasets $\alg$ obtains $R^2$ exceeding or matching that of both $\adassp$ and $\dpsgd$. $\alg$ achieves this even though $\adassp$ receives non-private access to the true feature and label norms, and $\dpsgd$ receives non-private access to extensive hyperparameter tuning.

We briefly elaborate on the latter. Our experiments tune $\dpsgd$ over a large grid consisting of 2,184 joint hyperparameter settings, over \textbf{learning\_rate} $\in \{10^{-6},10^{-5},\dots,1\}$, \textbf{clip\_norm} $\in \{10^{-6},10^{-5},\dots,10^6\}$, \textbf{microbatches} $\in \{2^5,2^6,\dots,2^{10}\}$, and $\textbf{epochs} \in \{1,5,10,20\}$. Ignoring the extensive computational resources required to do so at this scale (100 trials of each of the 2,184 hyperparameter combinations, for each dataset), we highlight that even mildly suboptimal hyperparameters are sufficient to significantly decrease $\dpsgd$'s utility. \cref{fig:results} quantifies this by recording the $R^2$ obtained by the hyperparameters that achieved the highest and 90th percentile median $R^2$ during tuning. While the optimal hyperparameters consistently produce results competitive with or sometimes exceeding that of $\alg$, even the mildly suboptimal hyperparameters nearly always produce results significantly worse than those of $\alg$. The exact hyperparameters used appear in \cref{subsec:accuracy_extended}.

We conclude our discussion of $\dpsgd$ by noting that it has so far omitted any attempt at differentially private hyperparameter tuning. We suggest that the results here indicate that any such method will need to select hyperparameters with high accuracy while using little privacy budget, and emphasize that the presentation of $\dpsgd$ in our experiments is generous.

\arxiv{\begin{figure}
    \centering
    \includegraphics[scale=0.5]{beijing_nondp.pdf}
    \includegraphics[scale=0.5]{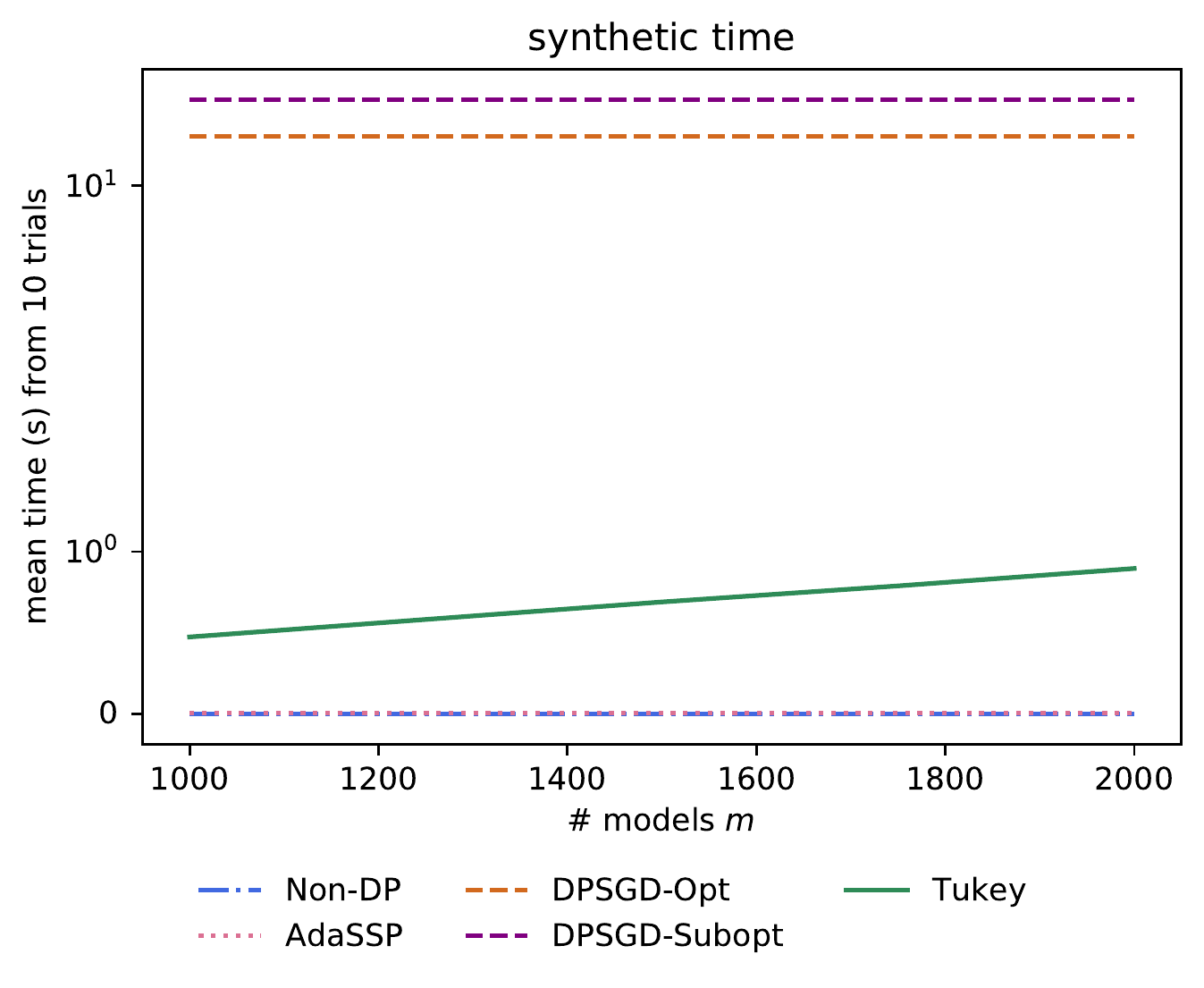}
    \caption{Left: plot of median $R^2$ obtained by $\nondp$ as the number of random samples from Beijing it uses varies. Right: plot of average time in seconds as the number of models $m$ used by $\alg$ varies}
    \label{fig:beijing_and_time}
\end{figure}}

Overall, $\alg$'s overall performance on the eight datasets is strong. We propose that the empirical evidence is enough to justify $\alg$ as a first-cut method for linear regression problems whose data dimensions satisfy its heuristic requirements ($n \gtrsim 1000 \cdot d$).

\subsection{Time Comparison}
\label{subsec:time}
We conclude with a brief discussion of runtime. The rightmost plot in Figure~\ref{fig:heuristic_and_time} records the average runtime in seconds over 10 trials of each method. $\alg$ is slower than the covariance matrix-based methods $\nondp$ and $\adassp$, but it still runs in under one second, and it is substantially faster than $\dpsgd$. $\alg$'s runtime also, as expected, depends linearly on the number of models $m$. Since the plots are essentially identical across datasets, we only include results for the Synthetic dataset here. Finally, we note that, for most reasonable settings of $m$, $\alg$ has runtime asymptotically identical to that of $\nondp$ (\cref{thm:main_approx}). The gap in practical performance is likely a consequence of the relatively unoptimized nature of our implementation.

\narxiv{\begin{figure}
    \centering
    \includegraphics[scale=0.5]{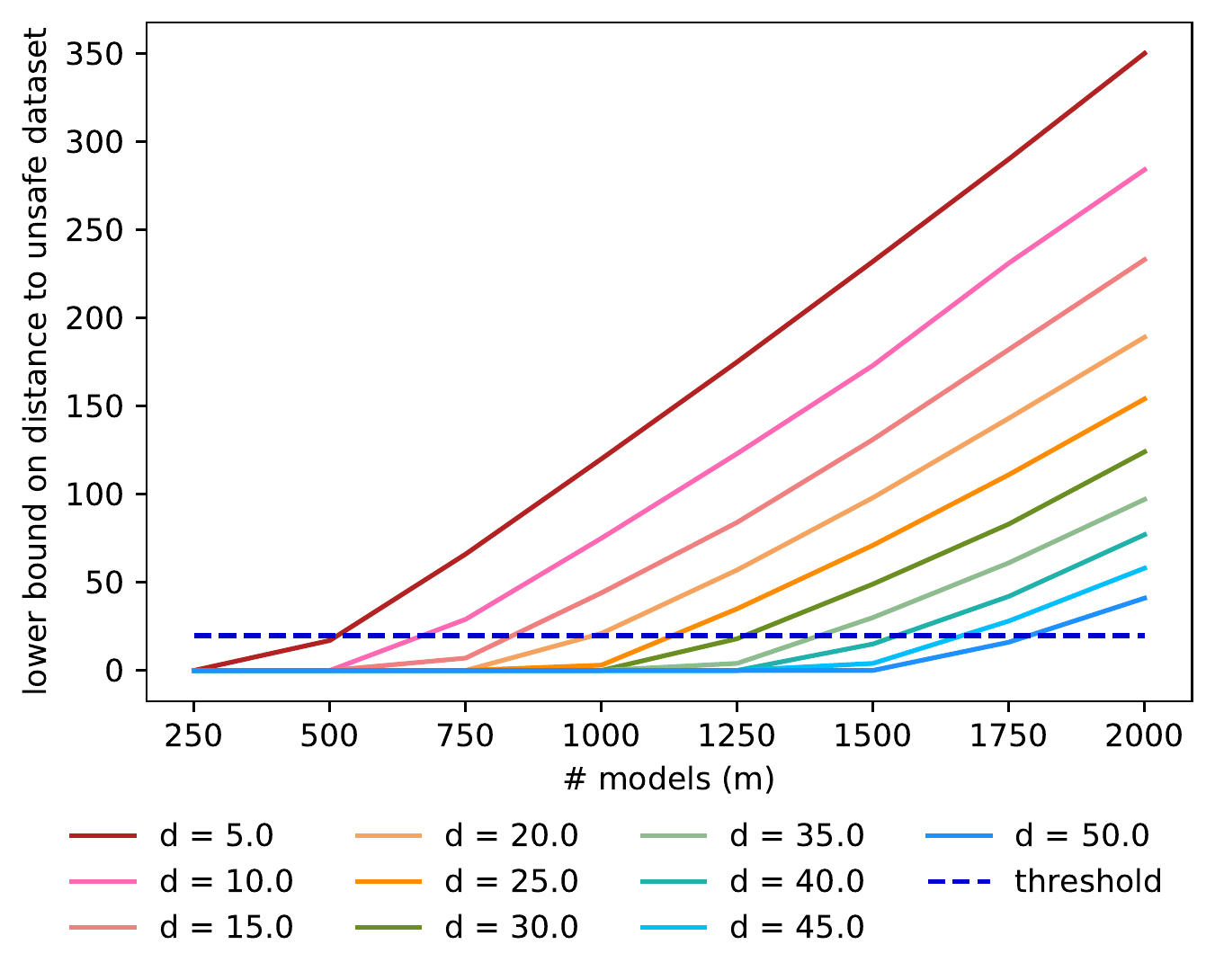}
    \includegraphics[scale=0.5]{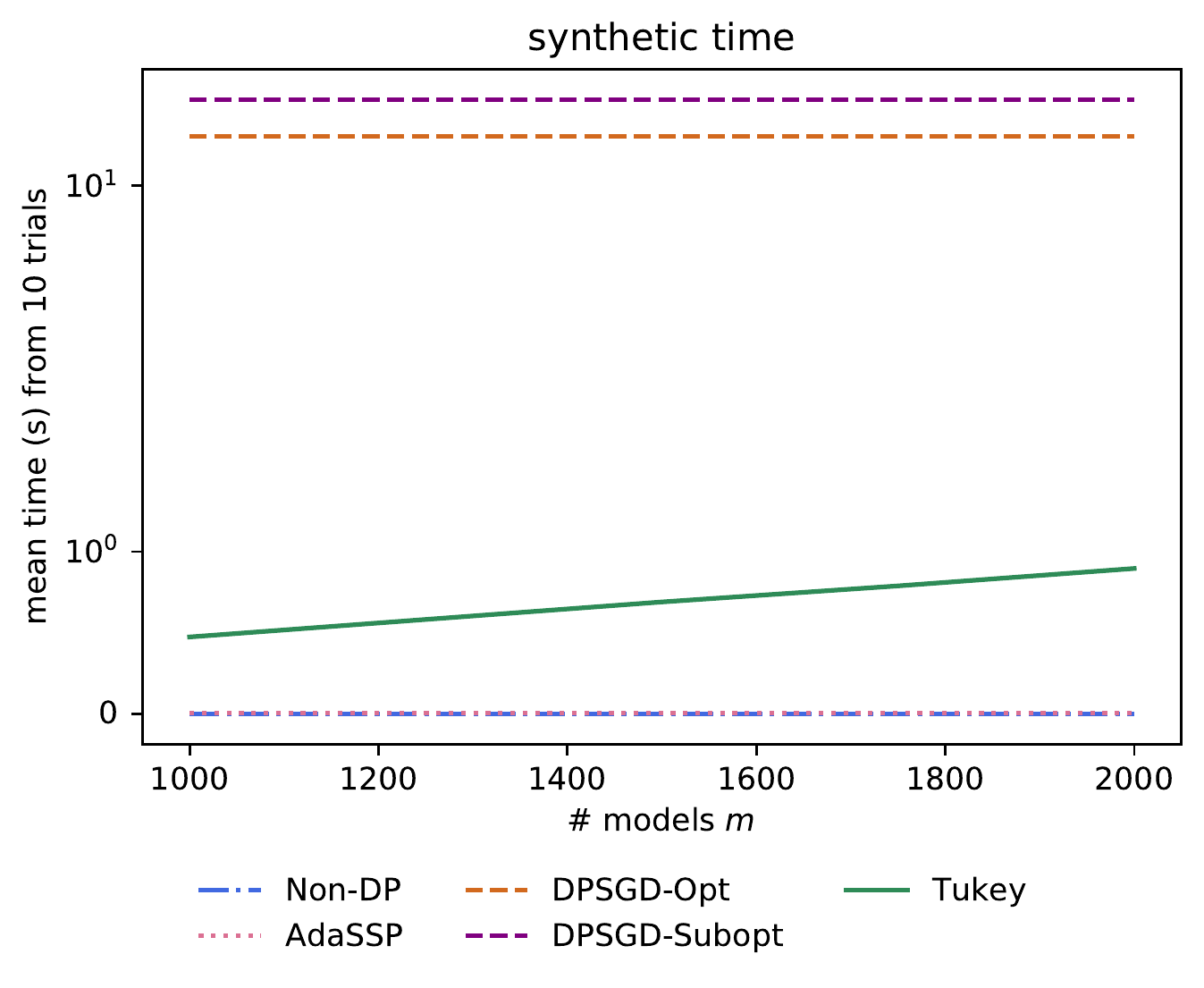}
    \caption{Left: plot of Hamming distance to unsafety using \cref{lem:ptr_2} as the feature dimension $d$ and number of models $m$ varies, using $(\ln(3), 10^{-5})$-DP and $n = (d+1)m$ throughout. Right: plot of average time in seconds as the number of models $m$ used by $\alg$ varies}
    \label{fig:heuristic_and_time}
\end{figure}}

\section{Future Directions}
\label{sec:conclusion}
An immediate natural extension of $\alg$ would generalize the approach to  similar problems such as logistic regression. More broadly, while this work focused on linear regression for the sake of simplicity and wide applicability, the basic idea of $\alg$ can in principle be applied to select from arbitrary non-private models that admit expression as vectors in $\mathbb{R}^d$. Part of our analysis observed that $\alg$ may benefit from algorithms and data that lead to Gaussian-like distributions over models; describing the characteristics of algorithms and data that induce this property --- or a similar property that better characterizes the performance of $\alg$ --- is an open question.
\section{Acknowledgments}
\label{sec:acks}
We thank Gavin Brown for helpful discussion of~\citet{BGSUZ21}, and we thank Jenny Gillenwater for useful feedback on an early draft. We also thank attendees of the Fields Institute Workshop on Differential Privacy and Statistical Data Analysis for helpful general discussions.

\newpage

\bibliographystyle{iclr2023_conference}
\bibliography{references}

\begin{thebibliography}{37}
\providecommand{\natexlab}[1]{#1}
\providecommand{\url}[1]{\texttt{#1}}
\expandafter\ifx\csname urlstyle\endcsname\relax
  \providecommand{\doi}[1]{doi: #1}\else
  \providecommand{\doi}{doi: \begingroup \urlstyle{rm}\Url}\fi

\bibitem[Abadi et~al.(2016)Abadi, Chu, Goodfellow, McMahan, Mironov, Talwar,
  and Zhang]{ACGMMTZ16}
Martin Abadi, Andy Chu, Ian Goodfellow, H~Brendan McMahan, Ilya Mironov, Kunal
  Talwar, and Li~Zhang.
\newblock Deep learning with differential privacy.
\newblock In \emph{Conference on Computer and Communications Security (CCS)},
  2016.

\bibitem[Agarwal(2017)]{A17}
Shivam Agarwal.
\newblock Diamonds dataset.
\newblock \url{https://www.kaggle.com/datasets/shivam2503/diamonds}, 2017.
\newblock Accessed: 2022-7-6.

\bibitem[Alabi et~al.(2022)Alabi, McMillan, Sarathy, Smith, and
  Vadhan]{AMSSV22}
Daniel Alabi, Audra McMillan, Jayshree Sarathy, Adam Smith, and Salil Vadhan.
\newblock Differentially private simple linear regression.
\newblock In \emph{Privacy Enhancing Technologies Symposium (PETS)}, 2022.

\bibitem[Bassily et~al.(2014)Bassily, Smith, and Thakurta]{BST14}
Raef Bassily, Adam Smith, and Abhradeep Thakurta.
\newblock Private empirical risk minimization: Efficient algorithms and tight
  error bounds.
\newblock In \emph{Foundations of Computer Science (FOCS)}, 2014.

\bibitem[Bassily et~al.(2019)Bassily, Feldman, Talwar, and
  Guha~Thakurta]{BFTT19}
Raef Bassily, Vitaly Feldman, Kunal Talwar, and Abhradeep Guha~Thakurta.
\newblock Private stochastic convex optimization with optimal rates.
\newblock In \emph{Neural Information Processing Systems (NeurIPS)}, 2019.

\bibitem[Beimel et~al.(2019)Beimel, Moran, Nissim, and Stemmer]{BMNS19}
Amos Beimel, Shay Moran, Kobbi Nissim, and Uri Stemmer.
\newblock Private center points and learning of halfspaces.
\newblock In \emph{Conference on Learning Theory (COLT)}, 2019.

\bibitem[Brown et~al.(2021)Brown, Gaboardi, Smith, Ullman, and
  Zakynthinou]{BGSUZ21}
Gavin Brown, Marco Gaboardi, Adam Smith, Jonathan Ullman, and Lydia
  Zakynthinou.
\newblock Covariance-aware private mean estimation without private covariance
  estimation.
\newblock In \emph{Neural Information Processing Systems (NeurIPS)}, 2021.

\bibitem[Cai et~al.(2020)Cai, Wang, and Zhang]{CWZ20}
T~Tony Cai, Yichen Wang, and Linjun Zhang.
\newblock The cost of privacy: Optimal rates of convergence for parameter
  estimation with differential privacy.
\newblock \emph{Annals of Statistics}, 2020.

\bibitem[Chollet et~al.(2015)]{C15}
Fran\c{c}ois Chollet et~al.
\newblock Keras.
\newblock \url{https://keras.io}, 2015.

\bibitem[Cumings-Menon(2022)]{C22}
Ryan Cumings-Menon.
\newblock Differentially private estimation via statistical depth.
\newblock \emph{arXiv preprint arXiv:2207.12602}, 2022.

\bibitem[DSNY(2022)]{N22}
DSNY.
\newblock Dsny monthly tonnage data.
\newblock
  \url{https://data.cityofnewyork.us/City-Government/DSNY-Monthly-Tonnage-Data/ebb7-mvp5},
  2022.
\newblock Accessed: 2022-7-28.

\bibitem[Dwork et~al.(2006)Dwork, McSherry, Nissim, and Smith]{DMNS06}
Cynthia Dwork, Frank McSherry, Kobbi Nissim, and Adam Smith.
\newblock Calibrating noise to sensitivity in private data analysis.
\newblock In \emph{Theory of Cryptography Conference (TCC)}, 2006.

\bibitem[Dwork et~al.(2014)Dwork, Talwar, Thakurta, and Zhang]{DTTZ14}
Cynthia Dwork, Kunal Talwar, Abhradeep Thakurta, and Li~Zhang.
\newblock Analyze gauss: optimal bounds for privacy-preserving principal
  component analysis.
\newblock In \emph{Symposium on the Theory of Computing (STOC)}, 2014.

\bibitem[Google(2022)]{G22}
Google.
\newblock dp\_regression.
\newblock
  \url{https://github.com/google-research/google-research/tree/master/dp_regression},
  2022.

\bibitem[Johnson \& Preparata(1978)Johnson and Preparata]{J78}
David~S. Johnson and Franco~P Preparata.
\newblock The densest hemisphere problem.
\newblock \emph{Theoretical Computer Science}, 1978.

\bibitem[Kaplan et~al.(2020)Kaplan, Sharir, and Stemmer]{KSS20}
Haim Kaplan, Micha Sharir, and Uri Stemmer.
\newblock How to find a point in the convex hull privately.
\newblock In \emph{Symposium on Computational Geometry (SoCG)}, 2020.

\bibitem[Kifer et~al.(2012)Kifer, Smith, and Thakurta]{KST12}
Daniel Kifer, Adam Smith, and Abhradeep Thakurta.
\newblock Private convex empirical risk minimization and high-dimensional
  regression.
\newblock In \emph{Conference on Learning Theory (COLT)}, 2012.

\bibitem[Lauga(2022)]{L22}
Nathan Lauga.
\newblock Nba games dataset.
\newblock
  \url{https://www.kaggle.com/datasets/nathanlauga/nba-games?select=games.csv},
  2022.
\newblock Accessed: 2022-7-6.

\bibitem[Liu \& Talwar(2019)Liu and Talwar]{LT19}
Jingcheng Liu and Kunal Talwar.
\newblock Private selection from private candidates.
\newblock In \emph{Symposium on the Theory of Computing (STOC)}, 2019.

\bibitem[Liu et~al.(2021)Liu, Kong, Kakade, and Oh]{LKKO21}
Xiyang Liu, Weihao Kong, Sham Kakade, and Sewoong Oh.
\newblock Robust and differentially private mean estimation.
\newblock \emph{arXiv preprint arXiv:2102.09159}, 2021.

\bibitem[Liu et~al.(2022)Liu, Kong, and Oh]{LKO21}
Xiyang Liu, Weihao Kong, and Sewoong Oh.
\newblock Differential privacy and robust statistics in high dimensions.
\newblock In \emph{Conference on Learning Theory (COLT)}, 2022.

\bibitem[McSherry \& Talwar(2007)McSherry and Talwar]{MT07}
Frank McSherry and Kunal Talwar.
\newblock Mechanism design via differential privacy.
\newblock In \emph{Foundations of Computer Science (FOCS)}, 2007.

\bibitem[Medina \& Gillenwater(2020)Medina and Gillenwater]{MG21}
Andr{\'e}s~Mu{\~n}oz Medina and Jenny Gillenwater.
\newblock Duff: A dataset-distance-based utility function family for the
  exponential mechanism.
\newblock \emph{arXiv preprint arXiv:2010.04235}, 2020.

\bibitem[Milionis et~al.(2022)Milionis, Kalavasis, Fotakis, and
  Ioannidis]{MKFI22}
Jason Milionis, Alkis Kalavasis, Dimitris Fotakis, and Stratis Ioannidis.
\newblock Differentially private regression with unbounded covariates.
\newblock In \emph{Artificial Intelligence and Statistics (AISTATS)}, 2022.

\bibitem[Nugent(2017)]{N17}
Cam Nugent.
\newblock California housing prices dataset.
\newblock
  \url{https://www.kaggle.com/datasets/camnugent/california-housing-prices},
  2017.
\newblock Accessed: 2022-7-6.

\bibitem[NYSDOT(2013)]{N13}
NYSDOT.
\newblock New york weigh-in-motion station vehicle traffic counts dataset.
\newblock
  \url{https://data.ny.gov/Transportation/Weigh-In-Motion-Station-Vehicle-Traffic-Counts-201/gdpg-i86w},
  2013.
\newblock Accessed: 2022-7-6.

\bibitem[Papernot \& Steinke(2022)Papernot and Steinke]{PS22}
Nicolas Papernot and Thomas Steinke.
\newblock Hyperparameter tuning with renyi differential privacy.
\newblock In \emph{International Conference on Learning Representations
  (ICLR)}, 2022.

\bibitem[Pedregosa et~al.(2011)Pedregosa, Varoquaux, Gramfort, Michel, Thirion,
  Grisel, Blondel, Prettenhofer, Weiss, Dubourg, Vanderplas, Passos,
  Cournapeau, Brucher, Perrot, and Duchesnay]{PVGMT+11}
F.~Pedregosa, G.~Varoquaux, A.~Gramfort, V.~Michel, B.~Thirion, O.~Grisel,
  M.~Blondel, P.~Prettenhofer, R.~Weiss, V.~Dubourg, J.~Vanderplas, A.~Passos,
  D.~Cournapeau, M.~Brucher, M.~Perrot, and E.~Duchesnay.
\newblock Scikit-learn: Machine learning in {P}ython.
\newblock \emph{Journal of Machine Learning Research (JMLR)}, 2011.

\bibitem[Ramsay \& Chenouri(2021)Ramsay and Chenouri]{RC21}
Kelly Ramsay and Shoja'eddin Chenouri.
\newblock Differentially private depth functions and their associated medians.
\newblock \emph{arXiv preprint arXiv:2101.02800}, 2021.

\bibitem[ruiqurm(2018)]{R18}
ruiqurm.
\newblock Housing price in beijing.
\newblock \url{https://www.kaggle.com/datasets/ruiqurm/lianjia}, 2018.
\newblock Accessed: 2022-7-28.

\bibitem[Samaniego(2018)]{S18}
Antonio Samaniego.
\newblock Mlb games data from retrosheet.
\newblock
  \url{https://www.kaggle.com/datasets/samaxtech/mlb-games-data-from-retrosheet?select=game_log.csv},
  2018.
\newblock Accessed: 2022-7-28.

\bibitem[Sarathy et~al.(2022)Sarathy, Song, Haque, Schlatter, and
  Vadhan]{SSHSV22}
Jayshree Sarathy, Sophia Song, Audrey~Emma Haque, Tania Schlatter, and Salil
  Vadhan.
\newblock Don’t look at the data! how differential privacy reconfigures data
  subjects, data analysts, and the practices of data science.
\newblock In \emph{Theory and Practice of Differential Privacy (TPDP)}, 2022.

\bibitem[Sheffet(2019)]{S19}
Or~Sheffet.
\newblock Old techniques in differentially private linear regression.
\newblock In \emph{Algorithmic Learning Theory (ALT)}, 2019.

\bibitem[Theil(1992)]{T92}
Henri Theil.
\newblock A rank-invariant method of linear and polynomial regression analysis,
  1992.

\bibitem[Tukey(1975)]{T75}
John~W Tukey.
\newblock Mathematics and the picturing of data.
\newblock In \emph{International Congress of Mathematicians (IMC)}, 1975.

\bibitem[Varshney et~al.(2022)Varshney, Thakurta, and Jain]{VTJ22}
Prateek Varshney, Abhradeep Thakurta, and Prateek Jain.
\newblock (nearly) optimal private linear regression for sub-gaussian data via
  adaptive clipping.
\newblock In \emph{Conference on Learning Theory (COLT)}, 2022.

\bibitem[Wang(2018)]{W18}
Yu-Xiang Wang.
\newblock Revisiting differentially private linear regression: optimal and
  adaptive prediction \& estimation in unbounded domain.
\newblock In \emph{Uncertainty in Artificial Intelligence (UAI)}, 2018.

\end{thebibliography}

\newpage

\section{Appendix}
\label{sec:appendix}

\subsection{Illustration of Exact and Approximate Tukey Depth}
\label{subsec:depth_illustration}
\begin{figure}[h]
    \centering
    \includegraphics[scale=0.5]{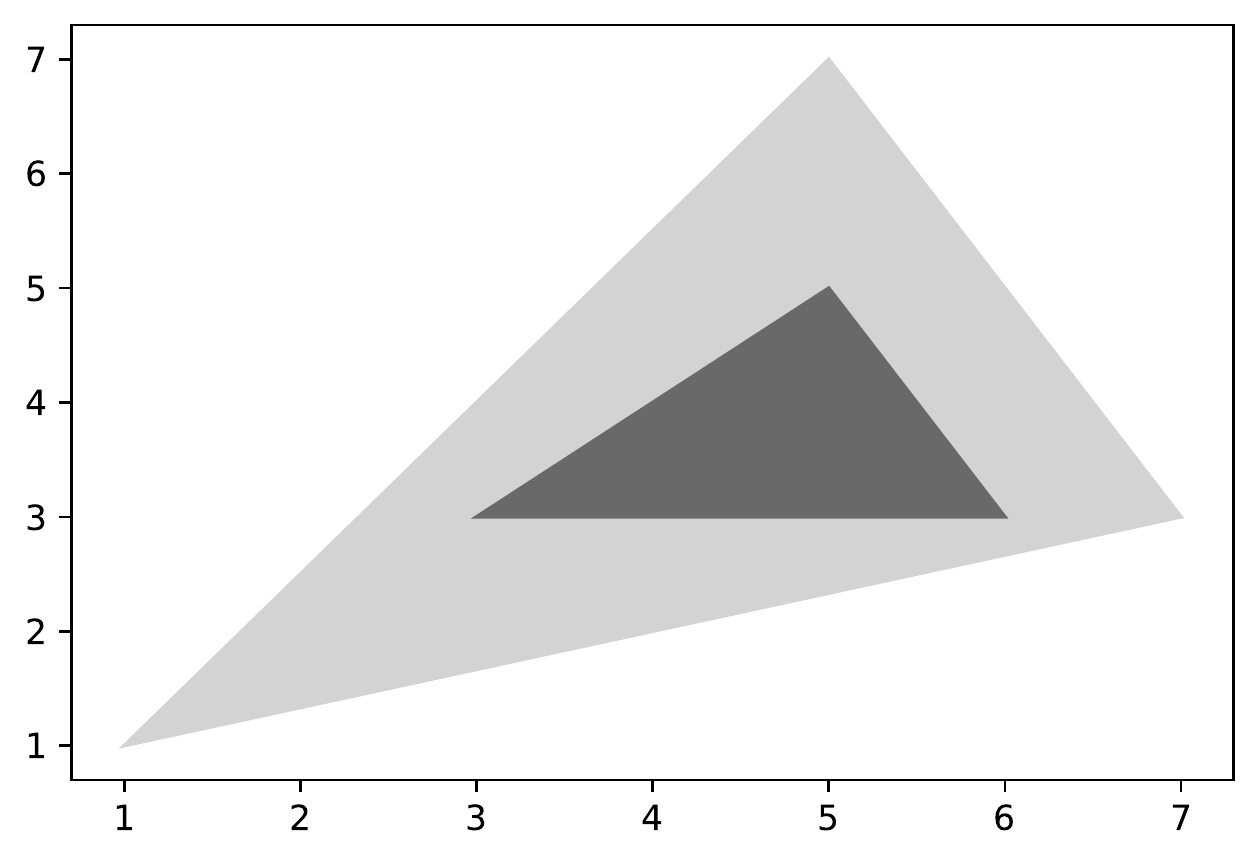}
    \includegraphics[scale=0.5]{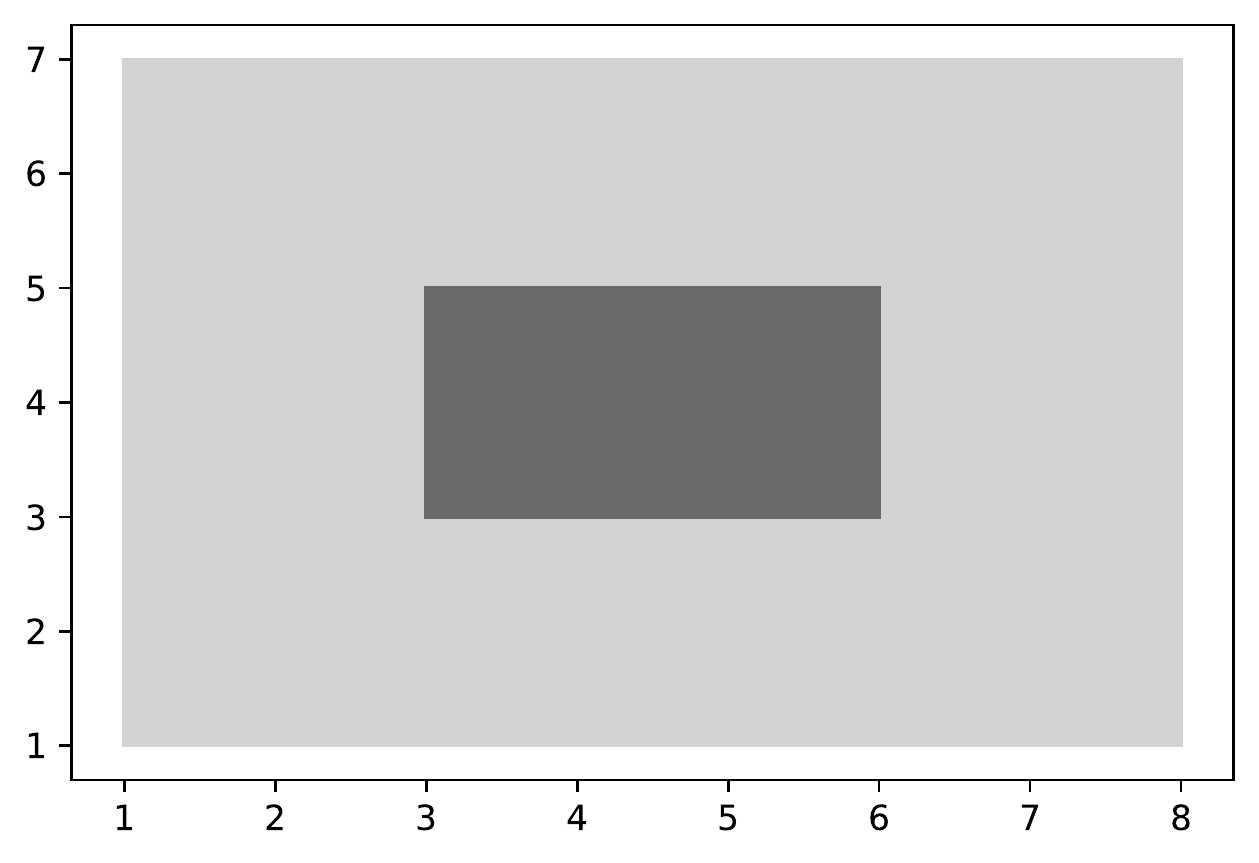}
    \caption{An illustrated comparison between exact (left) and approximate (right) Tukey depth. In both figures, the set of points is $\{(1, 1), (7, 3), (5, 7), (3, 3), (5, 5), (6, 3)\}$, the region of depth 0 is white, the region of depth 1 is light gray, and the region of depth 2 is dark gray. Note that for exact Tukey depth, the regions of different depths form a sequence of nested convex polygons; for approximate Tukey depth, they form a sequence of nested rectangles.}
    \label{fig:depth_example}
\end{figure}

\subsection{Omitted Proofs}
\label{subsec:omitted_proofs}

We start with the proof of our utility result for exact Tukey depth, \cref{thm:tukey-depth-convergence}.

\begin{proof}[Proof of \cref{thm:tukey-depth-convergence}]
This is a direct application of the results of~\citet{BGSUZ21}. They analyzed a notion of probabilistic (and normalized) Tukey depth over samples from a distribution:  $T_{\calN(\mu,\Sigma)}(y) := \min_{v}\P{X \sim \calN(\mu, \Sigma)}{\langle X, v \rangle \geq \langle y,v\rangle}$. Their Proposition 3.3 shows that $T_{\calN(\mu,\Sigma)}(y)$ can be characterized in terms of $\Phi$, the CDF of the standard one-dimensional Gaussian distribution. Specifically, they show $T_{\calN(\mu,\Sigma)}(y) = \Phi(-\|y-\mu \|_{\Sigma})$. From their Lemma 3.5, if $m\geq c \left(\frac{d+\log (1/\gamma)}{\alpha^2} \right)$, then with probability 
$1-\gamma$, $|p/m - T_{\calN(\beta^*,\Sigma)}(\hat{\beta})| \leq \alpha$. Thus
\begin{align*}
    -\alpha \leq&\ T_{\calN(\beta^*, \Sigma)}(\hat \beta) - p/m \\
    p/m - \alpha \leq&\ \Phi(-\|\hat \beta - \beta^*\|_\Sigma) \\
    p/m - \alpha \leq&\ 1 - \Phi(\|\hat \beta - \beta^*\|_\Sigma) \\
    \Phi(\|\hat \beta - \beta^*\|_\Sigma) \leq&\ 1 - p/m + \alpha \\
    \|\hat \beta - \beta^*\|_\Sigma \leq&\ \Phi^{-1}(1 - p/m + \alpha)
\end{align*}
where the third inequality used the symmetry of $\Phi$.
\end{proof}

Next, we prove our result about computing the volumes associated with different approximate Tukey depths, \cref{lem:compute_volumes}.

\begin{proof}[Proof of \cref{lem:compute_volumes}]
     By the definition of approximate Tukey depth, for arbitrary $y = (y_1, \ldots, y_d)$ of Tukey depth at least $i$, each of the $2d$ halfspaces $h_{y_1 \cdot e_1}, h_{y_1 \cdot -e_1}, \ldots, h_{y_d \cdot e_d}, h_{y_d \cdot -e_d}$ contains at least $i$ points from $D$, where $x \cdot y$ denotes multiplication of a scalar and vector. Fix some dimension $j \in [d]$.  Since $\min(|h_{y_j \cdot e_j} \cap D|, |h_{y_j \cdot -e_j} \cap D|) \geq i$, $y_j \in [S_{j, i}, S_{j, m-(i-1)}]$. Thus $V_{i, D} = \prod_{j=1}^d (S_{j, m-(i-1)} - S_{j, i})$. The computation of $S$ starting in Line~\ref{algln:compute_S} sorts $d$ arrays of length $m$ and so takes time $O(dm\log(m))$. Line~\ref{algln:volumes_start} iterates over $m/2$ depths and computes $d$ quantities, each in constant time, so its total time is $O(dm)$.
\end{proof}

The next result is our 1-sensitive and efficient adaptation of the lower bound from \citet{BGSUZ21}. We first restate that result. While their paper uses swap DP, the same result holds for add-remove DP. 

\begin{lemma}[Lemma 3.8~\citep{BGSUZ21}]
\label{lem:ptr}
    For any $k \geq 0$, if there exists a $g > 0$ such that $\frac{V_{t-k-1, D}}{V_{t+k+g+1, D}} \cdot e^{-\eps g/ 2} \leq \delta$, then for every database $z$ in $\unsafe{\eps, 4e^\eps\delta, t}$, $d_H(D, z) > k$, where $d_H$ denotes Hamming distance.
\end{lemma}

We now prove its adaptation, \cref{lem:ptr_2}.

\begin{proof}[Proof of \cref{lem:ptr_2}]
We first prove item 1. Let $D$ and $D'$ be neighboring databases, $D' = D \cup \{x\}$, and let $k^*_D$ and $k^*_{D'}$ denote the mechanism's outputs on the respective databases. It suffices to show $|k^*_D - k^*_{D'}| \leq 1$.

Consider some $V_{p}$ for nonnegative integer $p$. If $x$ has depth less than $p$ in $D$, then $V_{p-1, D} \geq V_{p, D'} \geq V_{p,D} > V_{p+1,D}$. Otherwise, $V_{p+1, D} < V_{p, D} = V_{p, D'} < V_{p-1, D}$. In either case,
\begin{equation}
\label{eq:Vs}
    V_{p+1, D} < V_{p, D} \leq V_{p, D'} \leq V_{p-1, D}.
\end{equation}

Now suppose there exist $k^*_{D'} \geq 0$ and $g^*_{D'} > 0$ such that
$\tfrac{V_{t-k^*_{D'}-1,D'}}{V_{t+k^*_{D'} + g^*_{D'}+1,D'}} \cdot e^{-\eps g^*_{D'}/2} \leq \delta$. Then by \cref{eq:Vs}, $\tfrac{V_{t-k^*_{D'},D}}{V_{t+k^*_{D'} + g^*_{D'},D}} \cdot e^{-\eps g^*_{D'}/2} \leq \delta$, so $k^*_D \geq k^*_{D'} -1$. Similarly, if there exist $k^*_D \geq 0$ and $g^*_D > 0$ such that
$\tfrac{V_{t-k^*_D-1,D}}{V_{t+k^*_D + g^*_D+1,D}} \cdot e^{-\eps g^*_D/2} \leq \delta$, then by \cref{eq:Vs}, $\tfrac{V_{t-k^*_D,D'}}{V_{t+k^*_D + g^*_D,D'}} \cdot e^{-\eps g^*_{D'}/2} \leq \delta$, so $k^*_{D'} \geq k^*_D -1$. Thus if $k^*_D \geq 0$ or $k^*_{D'} \geq 0$, $|k^*_D - k^*_{D'}| \leq 1$. The result then follows since $k^* \geq -1$.

We now prove item 2. This holds for $k \geq 0$ by Lemma~\ref{lem:ptr}; for $k = -1$, the lower bound on distance to unsafety is the trivial one, $d_H(D, z) \geq 0$.
\end{proof}

The next proof is for \cref{lem:ptr_dp}, which verifies the overall privacy and runtime of $\ptr$.

\begin{proof}[Proof of \cref{lem:ptr_dp}]
    The privacy guarantee follows from the 1-sensitivity of computing $k$ (\cref{lem:ptr_2}). For the runtime guarantee, we perform a binary search over the distance lower bound $k$ starting from the maximum possible $\nicefrac{m}{4}$ and, for each $k$, an exhaustive search over $g$. Note that if some $k, g$ pair satisfies the inequality in Lemma~\ref{lem:ptr}, there exists some $g'$ for every $k' < k$ that satisfies it as well. Thus since both have range $\leq \nicefrac{m}{4}$, the total time is $O(m\log(m))$.
\end{proof}

\subsection{Using Monotonicity}
\label{subsec:monotone}
This section discusses our use of monotonicity in the restricted exponential mechanism. \cref{def:em} states that, if $u$ is monotonic, the exponential mechanism can sample an output $y$ with probability proportional to $\exp\left(\frac{\eps u(D,y)}{\Delta_u}\right)$ and satisfy $\eps$-DP. Approximate Tukey depth is monotonic, so our application can also sample from this distribution. It remains to incorporate monotonicity into the PTR step.

It suffices to show that \cref{lem:ptr} also holds for a restricted exponential mechanism using a monotonic score function. Turning to the proof of \cref{lem:ptr} given by~\citet{BGSUZ21}, it suffices to prove their Lemma 3.7 using $w_x(S) = \int_S\exp(\eps q(x;y))dy$. Note that their $w_x(S)$ differs by the 2 in its denominator inside the exponent term; this modification is where we incorporate monotonicity. This difference shows up in two places in their argument. First, we can replace their bound
\[
    \frac{\P{}{M_{\eps, t}(x)=y}}{\P{}{M_{\eps, t}(x')=y}} \leq e^{\eps/2} \cdot \frac{w_{x'}(Y_{t,x'})}{w_x(Y_{t,x})} \leq e^\eps \cdot \frac{w_{x}(Y_{t,x'})}{w_x(Y_{t,x})}
\]
with the two cases that arise in add-remove differential privacy. The first considers $x' \subsetneq x$ and yields
\[
    \frac{\P{}{M_{\eps, t}(x)=y}}{\P{}{M_{\eps, t}(x')=y}} \leq e^\eps \cdot \frac{w_{x'}(Y_{t,x'})}{w_x(Y_{t,x})} \leq e^\eps \cdot \frac{w_{x}(Y_{t,x'})}{w_x(Y_{t,x})}
\]
since the mechanism on $x$ never assigns a lower score to an output than on $x'$. Using the same logic, the second considers $x \subsetneq x'$, and we get
\[
    \frac{\P{}{M_{\eps, t}(x)=y}}{\P{}{M_{\eps, t}(x')=y}} \leq  \frac{w_{x'}(Y_{t,x'})}{w_x(Y_{t,x})} \leq e^\eps \cdot \frac{w_{x}(Y_{t,x'})}{w_x(Y_{t,x})}.
\]
The second application in their argument, which bounds $\frac{\P{}{M_{\eps, t}(x')=y}}{\P{}{M_{\eps, t}(x)=y}}$, uses the same logic. As a result, their Lemma 3.7 also holds for a monotonic restricted exponential mechanism, and we can drop the 2 in the sampling distribution as desired.

\subsection{Sampling From a Region Details}
\label{subsec:sample_region_details}
We start by formally defining our partition.

\begin{definition}
\label{def:partition}
Given $d$-dimensional database $D$ and dimension $j \in [d]$, for $y \in \mathbb{R}^d$, let $T_{D, j}(y)$ denote the exact (one-dimensional) Tukey depth of point $y$ with respect to dimension $j$ in database $D$. Let $B_i$ denote the region of points with approximate Tukey depth $i$. 
Define the partition $\{C_{j, i}\}_{j=1}^d$ of $B_i$ as the volume of points where depth $i$ occurs in dimension $j$ for the first time, i.e., 
    \[
        C_{j, i} := \{y \in \mathbb{R}^d \mid \min_{j' < j}  T_{D, j'}(y) > i \text{ and } T_{D, j}(y) = i \text{ and } \min_{j' > j} T_{D, j'}(y) \geq i\}.
    \]
\end{definition}

The partition is well defined because any point with approximate Tukey depth $i$ is in exactly one of the $C_{j, i}$ volumes. Each $C_{j,i}$ is also the Cartesian product of three sets: any $y \in C_{j,i}$ must have 1) depth strictly greater than $i$ in dimensions $1,...,j-1$, 2) depth $i$ in dimension $j$, and 3) depth at least $i$ in dimensions $j+1,...,d$. Being the Cartesian product of three sets, the total volume of $C_{j,i}$ can be computed as the product of the three corresponding volumes in lower dimensions. We will denote these by $V_{<j,i}, W_{j,i}, V_{> j, i}$, formalized below.

\begin{definition}
\label{def:volume_sampling}
    Given $d$-dimensional database $D$, dimension $j \in [d]$, and depth $i$, define 
    \begin{enumerate}
        \item $V_{j, i, D} = \vol(\{y_j \mid y \in \mathbb{R}^d, \tilde T_D(y) \geq i\})$, the total length in dimension $j$ of the region with approximate Tukey depth at least $i$.
        \item $W_{j, i, D} = \vol(\{y_j \mid y \in \mathbb{R}^d, T_{D, j}(y) = i \text{ and } \tilde T_D(y) \geq i\})$, the total length in dimension $j$ of the region with depth exactly $i$ in dimension $j$ and approximate Tukey depth at least $i$.
        \item $V_{<j, i, D} = \vol(\{y_{1:j-1} \mid y \in \mathbb{R}^d, \tilde T_D(y) \geq i\})$, the volume of the projection onto the first $j-1$ dimensions of points with approximate Tukey depth at least $i$. Define $V_{>j, i, D}$ analogously.
    \end{enumerate}
    When $D$ is clear from context, we drop it from the subscript.
\end{definition}

The next lemma shows how to compute these and other relevant volumes. We again note that we fix $m$ to be even for neatness. The odd case is similar.

\begin{lemma}
\label{lem:compute_uniform_volumes}
    Given matrix $S \in \mathbb{R}^{d \times m}$ of projected and sorted models, as in Line~\ref{algln:S} of Algorithm~\ref{alg:main}, 
    \begin{enumerate}
        \item $V_{j, i} = S_{j, m-(i-1)} - S_{j, i}$,
        \item $W_{j, i} = V_{j, i} - V_{j, i+1}$,
        \item $V_{<j, i} = \prod_{j'=1}^{j-1} V_{j, i} $ and $ V_{>j, i} = \prod_{j'=j+1}^{d} V_{j, i}$
    \end{enumerate}
\end{lemma}
\begin{proof}
    The proofs of the first item uses essentially the same reasoning as the proof of Lemma~\ref{lem:compute_volumes}. For the second item, any point contributing to $V_{j, i}$ but not $V_{j, i+1}$ has depth exactly $i$ in dimension $j$. For the third item, a point contributes to $V_{<j, i}$ if and only if it has depth at least $i$ in all $d$ dimensions; since the resulting region is a rectangle, its volume is the product of its side lengths.
\end{proof}
With Lemma~\ref{lem:compute_uniform_volumes}, we can now prove that $\samplepoint$ works as intended.
\begin{proof}[Proof of \cref{lem:samplepoint}]
    Given $S \in \mathbb{R}^{d \times m}$, define $B_i$ and $\{C_{j,i} \}^d_{j=1}$ as in \cref{def:partition}. To show the outcome of $\samplepoint(S, i)$ is uniformly distributed over points with approximate Tukey depth $i$, it suffices to show the algorithm samples a $C_{j,i}$ with probability proportional to its volume. Recall that $C_{j, i}$ is the set of points with depth greater than $i$ in dimensions $1, 2, \ldots, j-1$, exactly $i$ in dimension $j$, and at least $i$ in the remaining dimensions. With this interpretation, $C_{j,i}$ is the Cartesian product of three lower dimensional regions, and thus its volume is the product of the corresponding volumes,  
    \[
        \vol(C_{j,i}) = V_{< j, i+1} \cdot W_{j, i} \cdot V_{> j, i} .
    \]
    These quantities, along with the normalizing constant $V_{\geq 1, i}$, can be computed using \cref{lem:compute_uniform_volumes}.
    
    Since $i$ is fixed, computing the full set of $V_{j, i}$ and $W_{j, i}$ takes time $O(d)$, and by tracking partial sums and using logarithms, we can compute the full set of $V_{<j, i}$ and $V_{>j, i}$ in time $O(d)$ as well. The last step is sampling the final point $y$, which takes time $O(d)$ using the previously computed $S$.
\end{proof}

\begin{algorithm}
    \begin{algorithmic}[1]
       \STATE {\bfseries Input:} Tukey depth region volumes $V$, sorted collection of estimators $S$, depth restriction $t$, privacy parameter $\eps$ \alglinelabel{algln:rtukeyem_input}
       \FOR{$i = t, t+1, \ldots, |V|-1$}
        \STATE Compute volume of region of Tukey depth exactly $i$, $W_i \gets V_i - V_{i+1}$
       \ENDFOR
       \STATE Sample depth $\hat i$ from distribution where $\P{}{i} \propto W_i \exp\left(\frac{\eps \cdot i}{2}\right)$
       \STATE Return $y \gets \samplepoint(S, \hat i)$
    \end{algorithmic}
    \caption{$\rtukeyem$}
    \label{alg:rtukeyem}
\end{algorithm}

\begin{algorithm}
    \begin{algorithmic}[1]
        \STATE {\bfseries Input:} Sorted collection of estimators $S$, depth $i$ \alglinelabel{algln:samplepoint_input}
        \STATE $d, m \gets $ number of rows and columns in $S$
        \STATE Compute $V_{\geq 1, i}$ using Lemma~\ref{lem:compute_uniform_volumes}
        \FOR{$j = 1, \ldots, d$}
            \STATE Compute $V_{<j,i+1}$, $W_{j,i}$, and $V_{> j, i}$ using Lemma~\ref{lem:compute_uniform_volumes}
            \STATE Compute $\vol(C_{j, i}) = V_{<j, i+1} \cdot W_{j, i} \cdot V_{>j, i}$
        \ENDFOR
        \STATE Sample index $j^* \in [d]$ with probability $\frac{\vol(C_{j,i})}{V_{\geq 1, i}}$ 
        \FOR{$j' = 1, \ldots, j^*-1$}
            \STATE $y_{j'} \gets$ uniform random sample from $[S_{j',i+1}, S_{j', m-i}]$
        \ENDFOR
        \STATE $y_{j^*} \gets$ uniform random sample from $[S_{j^*,i}, S_{j^*, i+1}) \cup (S_{j^*, m-i}, S_{j^*, m-(i-1)}]$
        \FOR{$j' = j^*+1, \ldots, d$}
            \STATE $y_{j'} \gets$ uniform random sample from $[S_{j',i}, S_{j', m-(i-1)}]$
        \ENDFOR
       \STATE Return $y$
    \end{algorithmic}
    \caption{$\samplepoint$}
    \label{alg:samplepoint}
\end{algorithm}

\subsection{Dataset Feature Selection Details}
\label{subsec:features}
This section provides details for each of the real datasets evaluated in our experiments.

\begin{enumerate}
    \item \textbf{California Housing}~\cite{N17}: The label is \texttt{median\_housevalue}, and the categorical \texttt{ocean\_proximity} is dropped.
    \item \textbf{Diamonds}~\cite{A17}: The label is \texttt{price}. Ordinal categorical features (\texttt{carat}, \texttt{color}, \texttt{clarity}) are replaced with integers $1, 2, \ldots$.
    \item \textbf{Traffic}~\cite{N13}: The label is passenger vehicle count (\texttt{Class 2}), and the remaining features are motorcycles (\texttt{Class 1}) and pickups, panels, and vans (\texttt{Class 3}).
    \item \textbf{NBA}~\cite{L22}: The label is \texttt{PTS\_home}, and the features are \texttt{FT\_PCT\_home}, \texttt{FG3\_PCT\_home}, \texttt{FG\_PCT\_home}, \texttt{AST\_home}, and \texttt{REB\_home}.
    \item \textbf{Beijing Housing}~\cite{R18}: The label is \texttt{totalPrice}. Features are days on market (\texttt{DOM}), \texttt{followers}, area of house in meters (\texttt{square}), number of kitchens (\texttt{kitchen}), \texttt{buildingType}, \texttt{renovationCondition}, building material (\texttt{buildingStructure}), ladders per residence (\texttt{ladderRatio}), elevator presence \texttt{elevator}, whether previous owner owned for at least five years (\texttt{fiveYearsProperty}), proximity to subway (\texttt{subway}), \texttt{district}, and nearby average housing price (\texttt{communityAverage}). Categorical \texttt{buildingType}, \texttt{renovationCondition}, and \texttt{buildingStructure} are encoded as one-hot variables. We additionally removed a single outlier row (60422) whose norm is more than two orders of magnitude larger than that of other points; none of the DP algorithms achieved positive $R^2$ with the row included. An earlier version of this paper retained this outlier.
    \item \textbf{New York Garbage}~\cite{N22}: The label is \texttt{REFUSETONSCOLLECTED}. The features are \texttt{PAPERTONSCOLLECTED} and \texttt{MPGTONSCOLLECTED}. The categorical \texttt{BOROUGH} is encoded as one-hot variables.
    \item \textbf{MLB}~\cite{S18}: The label is home team runs (\texttt{h\_score}) and the features are \texttt{v\_strikeouts}, \texttt{v\_walks}, \texttt{v\_pitchers\_used},
    \texttt{v\_errors}, \texttt{h\_at\_bats}, \texttt{h\_hits}, \texttt{h\_doubles},
    \texttt{h\_triples}, \texttt{h\_homeruns}, and \texttt{h\_stolen\_bases}.
\end{enumerate}

\subsection{Extended Experiment Results}
\label{subsec:accuracy_extended}

\cref{fig:dpsgd_hyperparameters} records, respectively, the number of epochs, the clipping norm, the learning rate, and the number of microbatches for both the best and 90th percentile hyperparameter settings on each dataset.
\begin{figure}[!htb]
\begin{center}
\begin{tabular}{ |c|c|c| } 
 \hline
 Dataset & $\dpsgd$ (tuned) & $\dpsgd$ (90\% tuned)\\ 
 \hline \hline
 Synthetic & $(20, 1, 1, 128)$ & $(20, 10^{-3}, 0.1, 64)$ \\
 \hline
 California & $(20, 100, 1, 64)$ & $(10, 1000, 0.01, 64)$ \\ 
 \hline
 Diamonds & $(20, 10^6, 1, 128)$ & $(10, 10^6, 0.1, 32)$ \\
 \hline
 Traffic & $(1, 10^6, 1, 1024)$ & $(10, 10^5, 0.1, 512)$ \\
 \hline
 NBA & $(20, 100, 1, 512)$ & $(20, 10^{-6}, 0.1, 32)$ \\
 \hline
 Beijing & $(20, 100, 0.01, 512)$ & $(20, 0.1, 0.001, 512)$ \\
 \hline
 Garbage & $(20, 1, 1, 32)$ & $(5, 1000, 1, 64)$ \\
 \hline
 MLB &  $(10, 100, 0.01, 512)$ & $(10, 10^{-5}, 0.001, 128)$ \\
 \hline
\end{tabular}
\end{center}
\caption{Hyperparameter settings used by $\dpsgd$ on each dataset.} 
    \label{fig:dpsgd_hyperparameters}
\end{figure}

\begin{figure}[p]
    \centering
    \includegraphics[scale=0.48]{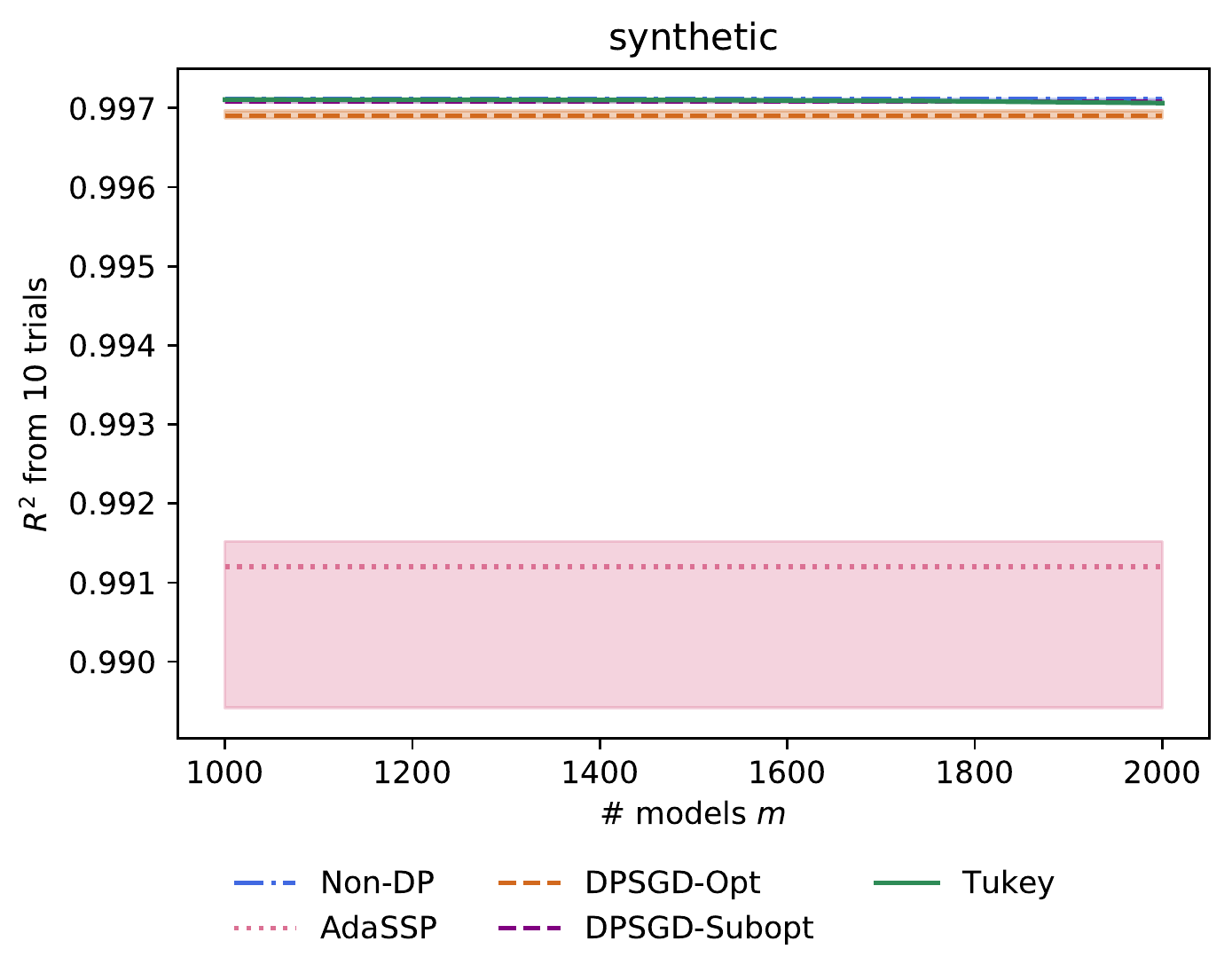}
    \includegraphics[scale=0.48]{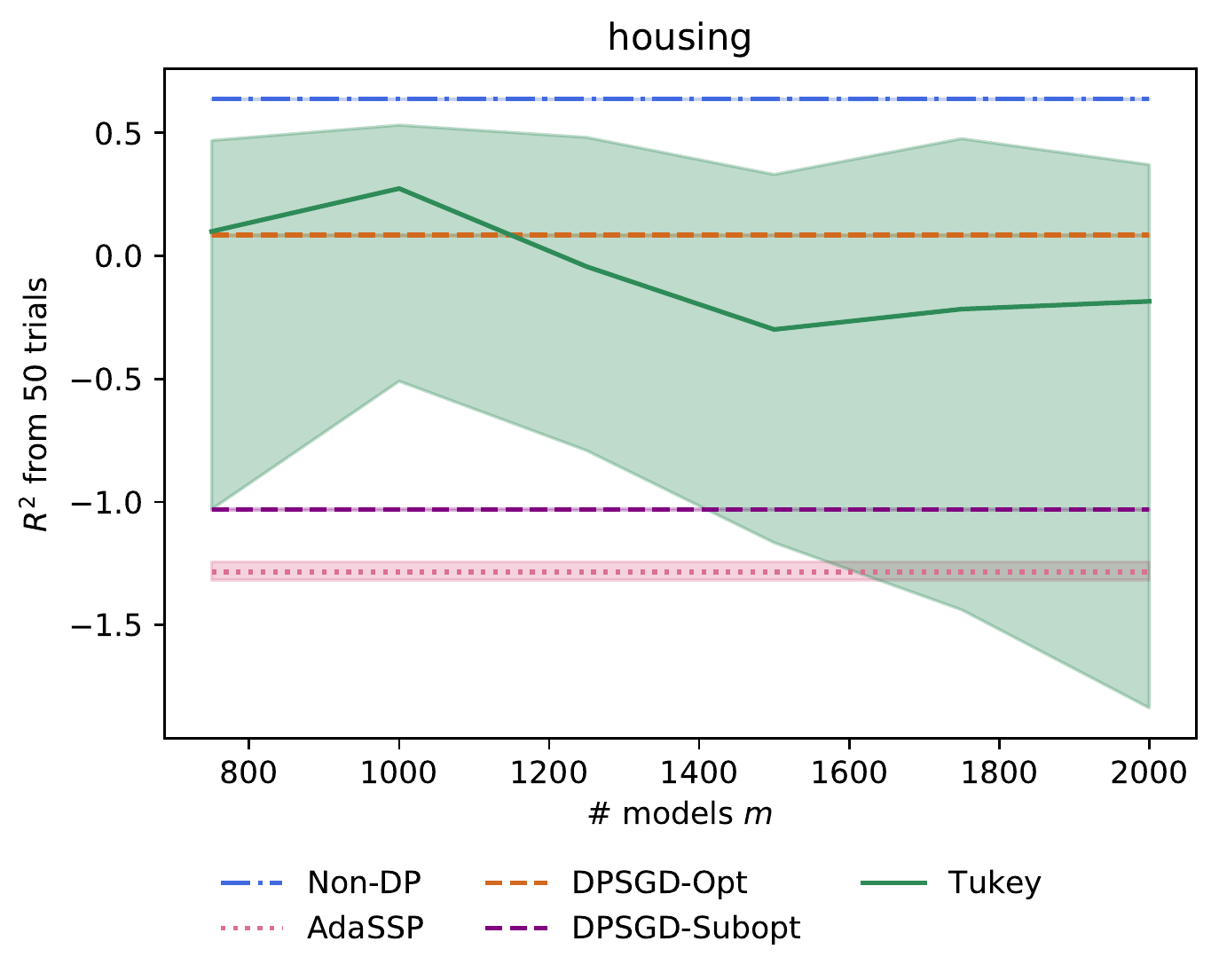}
    \includegraphics[scale=0.48]{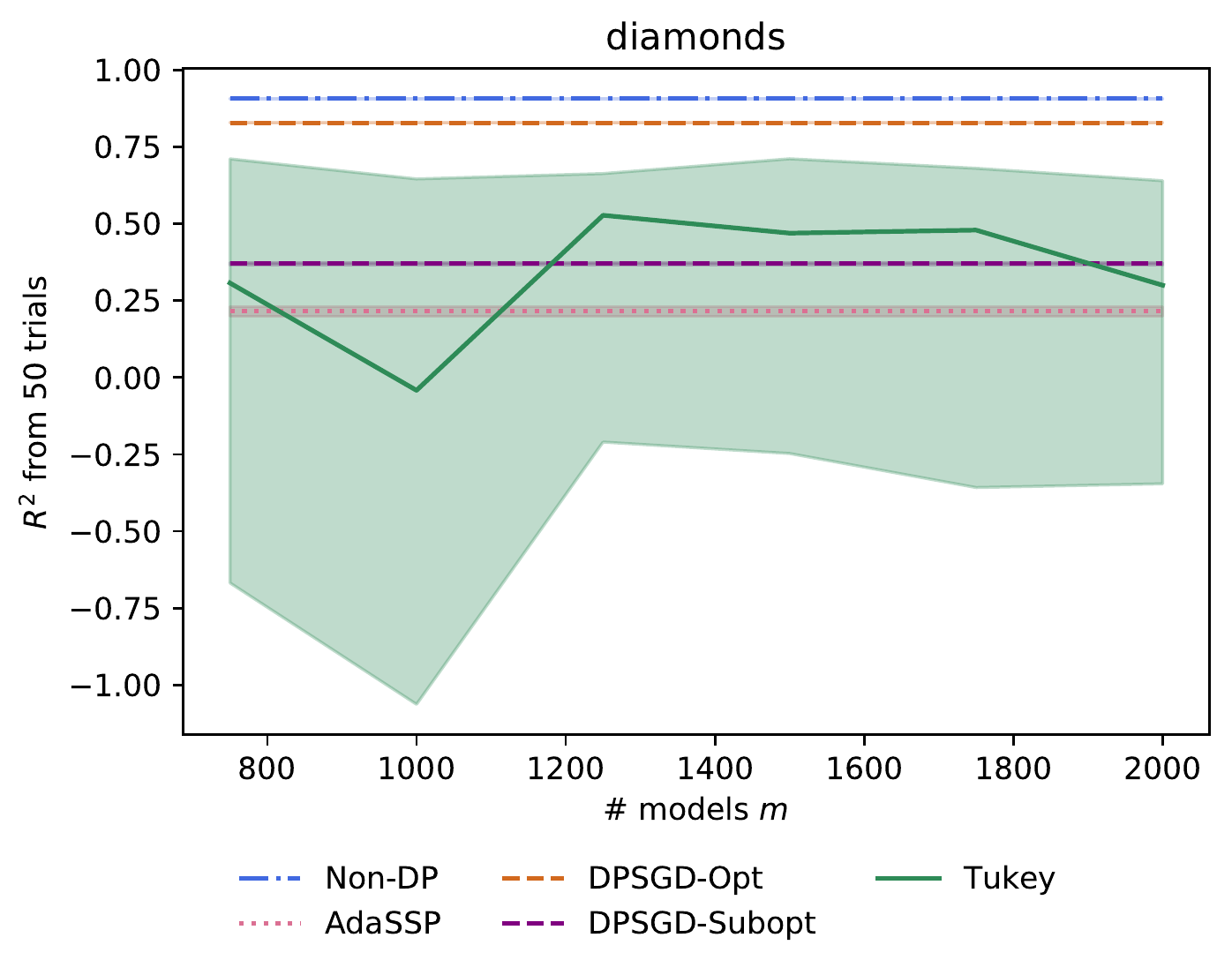}
    \includegraphics[scale=0.48]{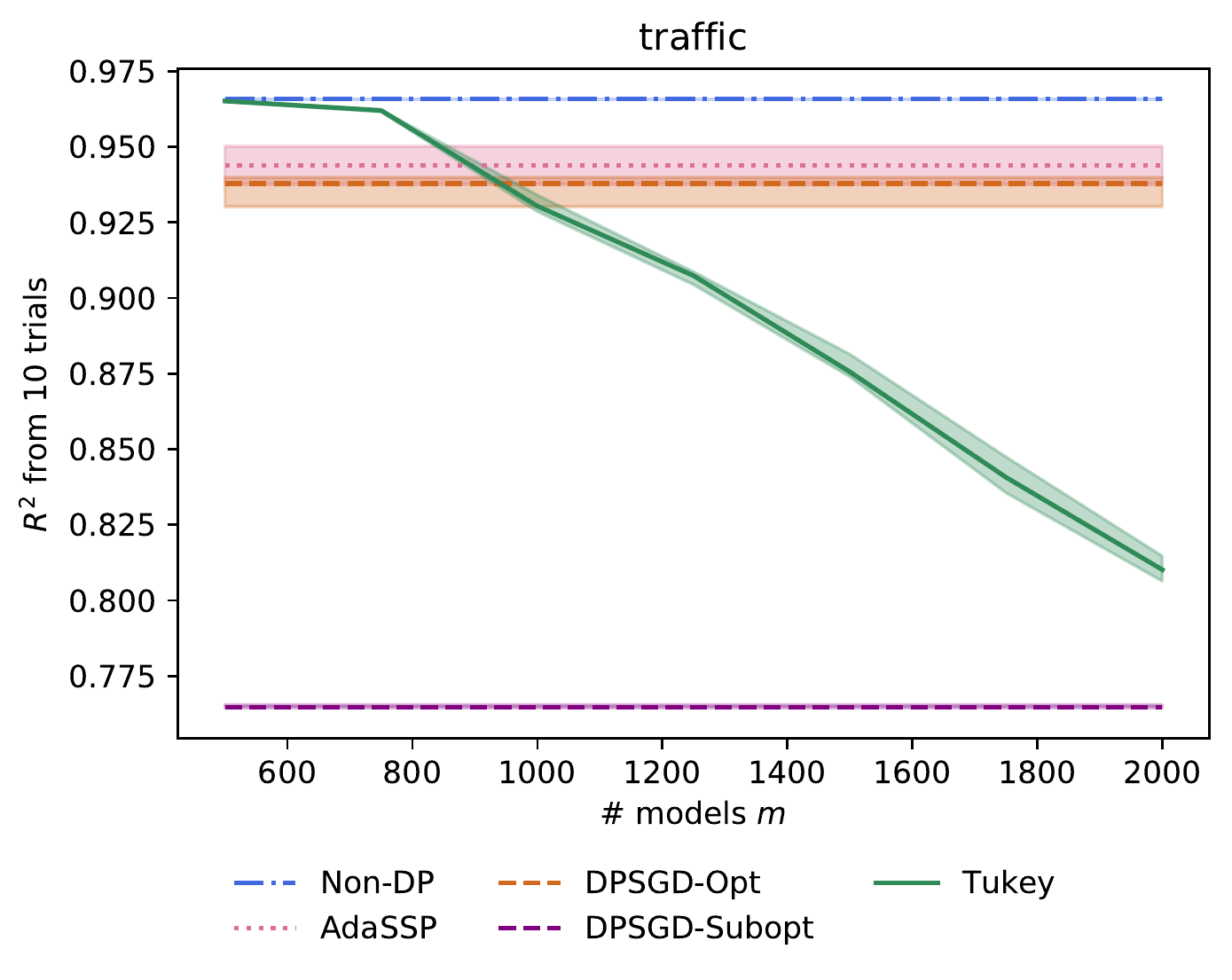}
    \includegraphics[scale=0.48]{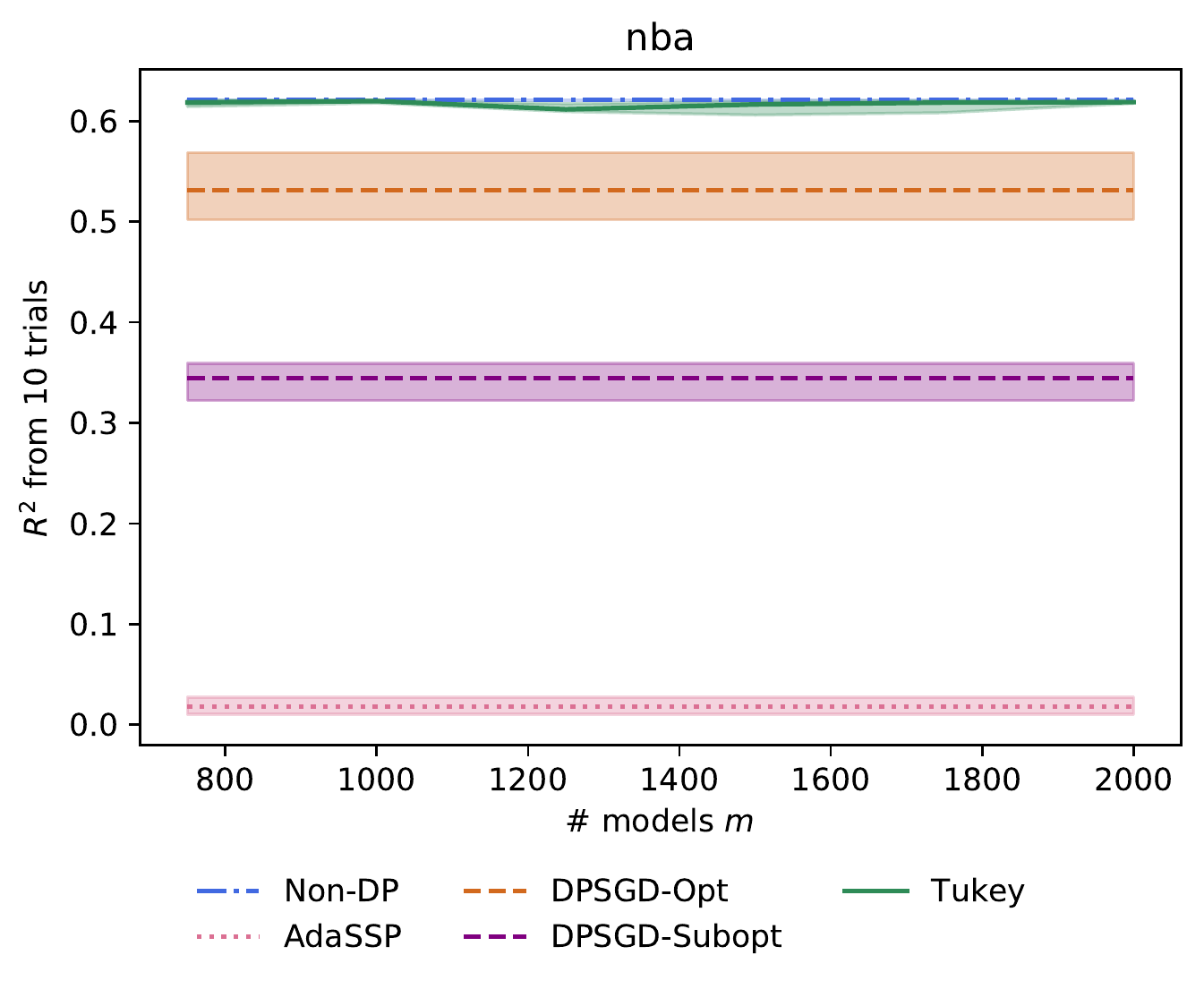}
    \includegraphics[scale=0.48]{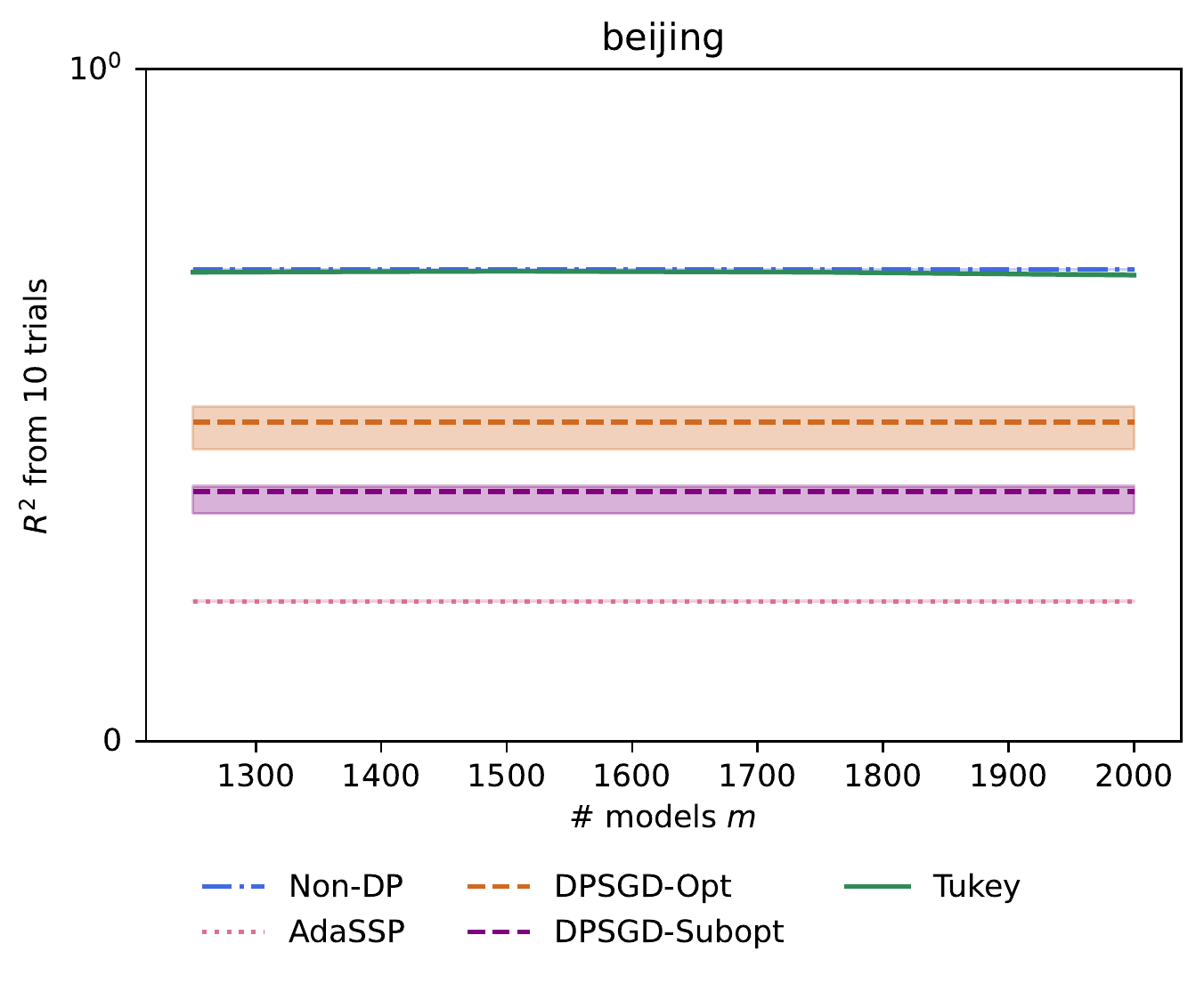}
    \includegraphics[scale=0.48]{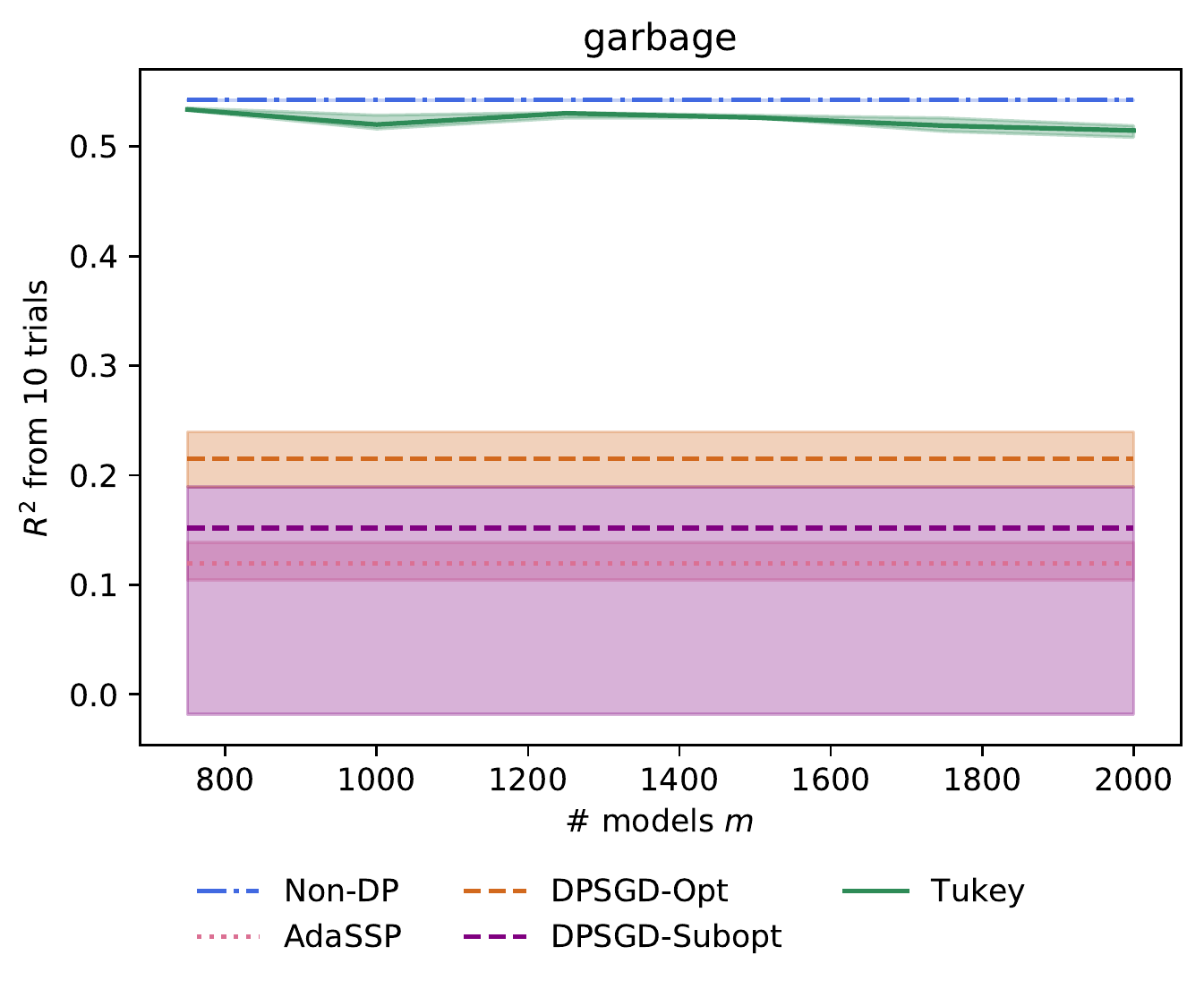}
    \includegraphics[scale=0.48]{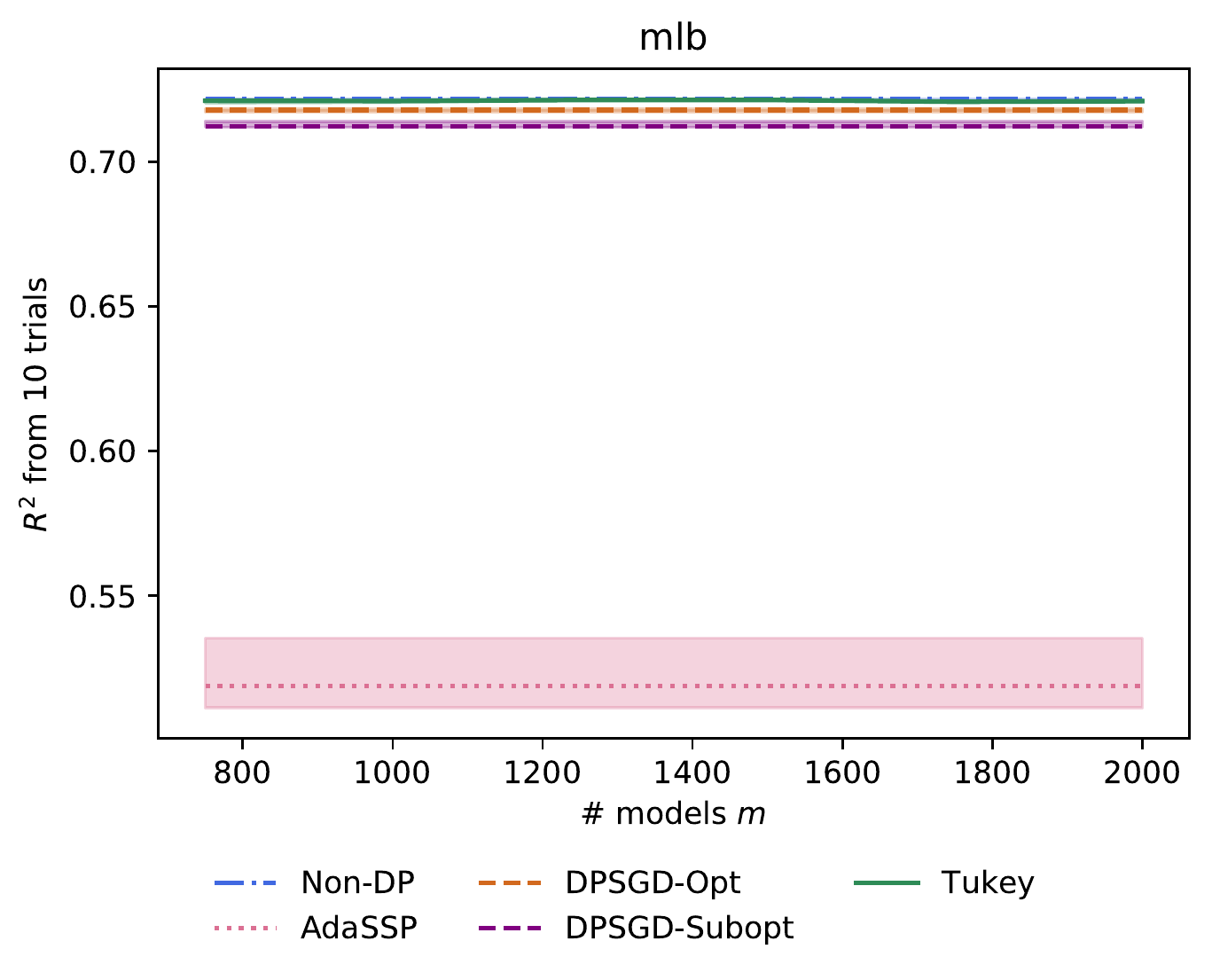}
    \caption{Plots of $R^2$ as the number of models $m$ used by $\alg$ varies. The lines mark medians and the shaded regions span the first and third quartiles. All datasets except Housing and Diamonds use 10 trials. Housing and Diamonds use 50 trials due to the variance of $\alg$. Methods other than $\alg$ appear as flat lines because they do not vary with $m$. Each plot varies the number of models $m$ in increments of 250, starting with the $m$ sufficient to pass PTR in all trials.}
    \label{fig:extended_results}
\end{figure}

\subsection{Distribution of models for all datasets}
\label{subsec:model_distributions}
Figures \ref{fig:hist-synthetic} through \ref{fig:hist-mlb} display histograms of models trained on each dataset. For each plot, we train 2,000 standard OLS models on disjoint partitions of the data and plot the resulting histograms for each coefficient. The red curve plots a Gaussian distribution with the same mean and standard deviation as the underlying data.

\begin{figure}
    \centering
    \includegraphics[width = \textwidth]{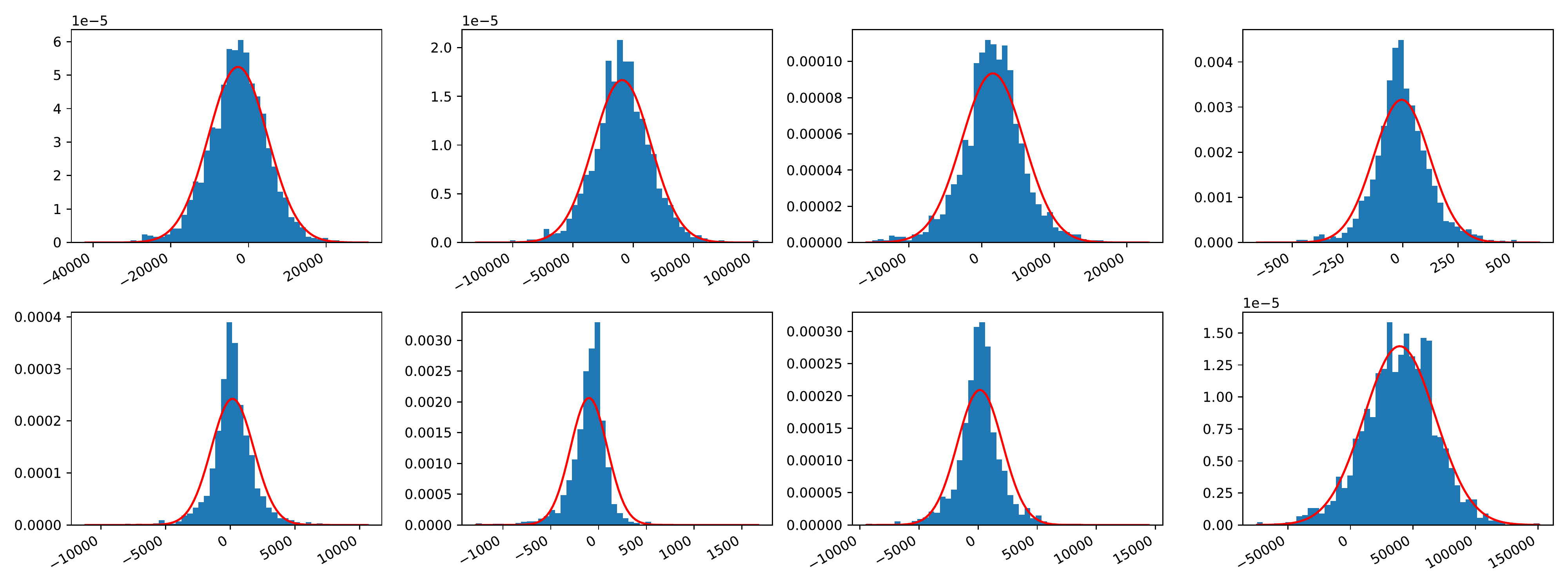}
    \caption{Histograms of models on the California dataset.}
    \label{fig:hist-housing}
\end{figure}

\begin{figure}
    \centering
    \includegraphics[width = \textwidth]{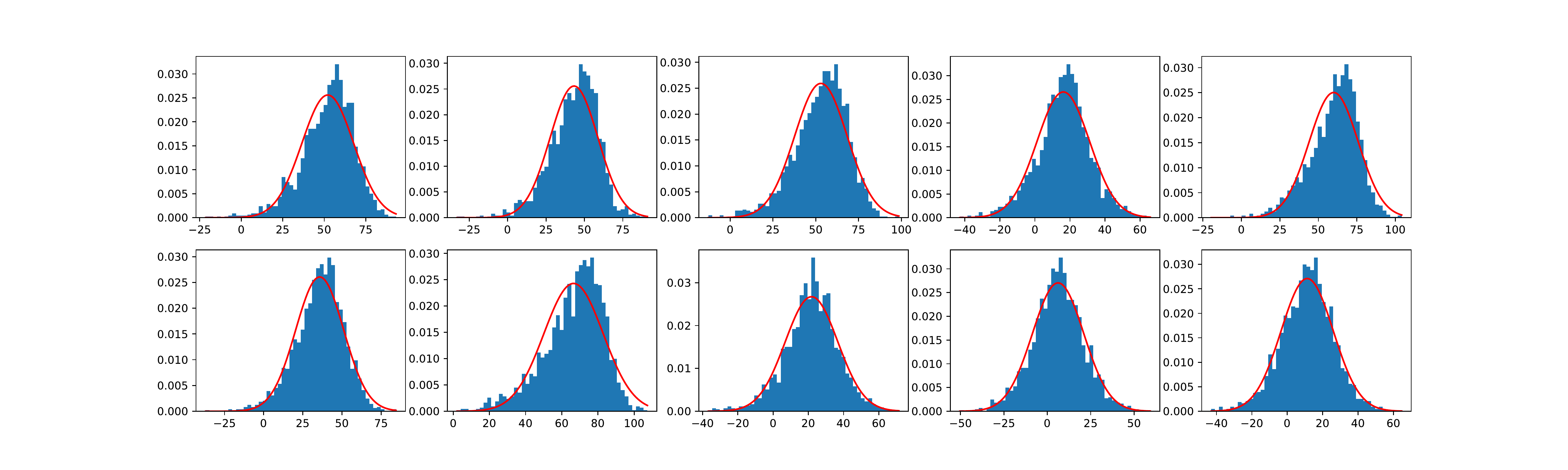}
    \caption{Histograms of models on Synthetic.}
    \label{fig:hist-synthetic}
\end{figure}

\begin{figure}
    \centering
    \includegraphics[width = \textwidth]{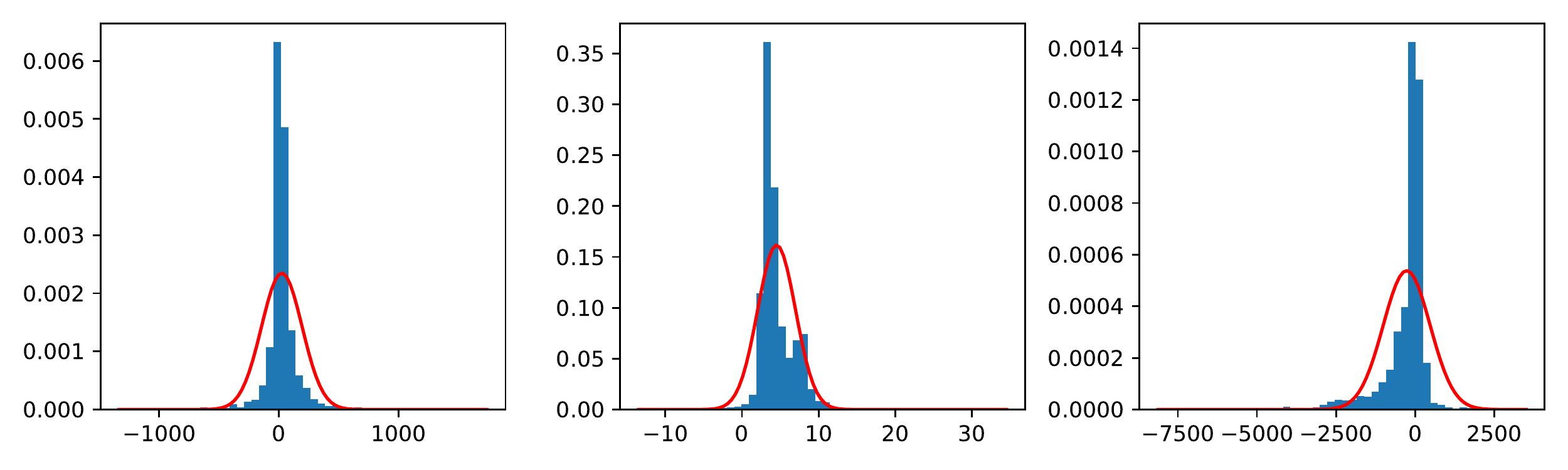}
    \caption{Histograms of models on Traffic.}
    \label{fig:hist-traffic}
\end{figure}

\begin{figure}
    \centering
    \includegraphics[width = \textwidth]{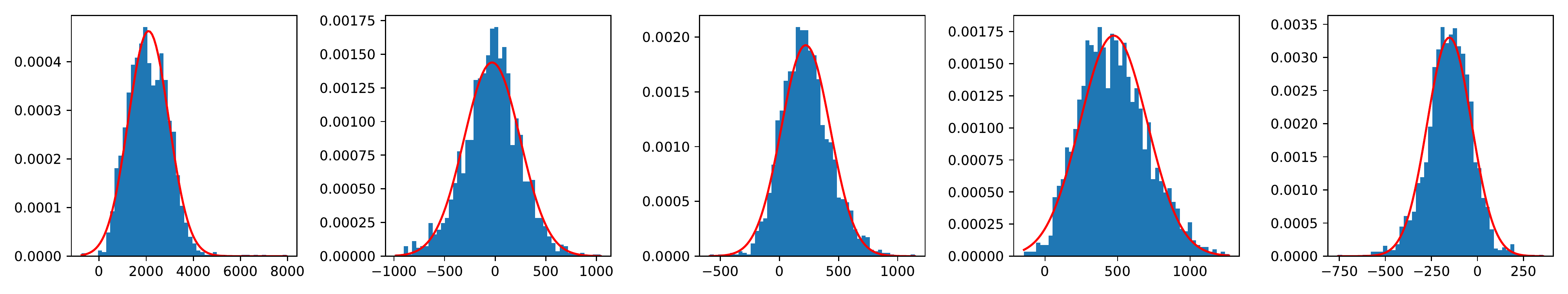}
    \includegraphics[width = \textwidth]{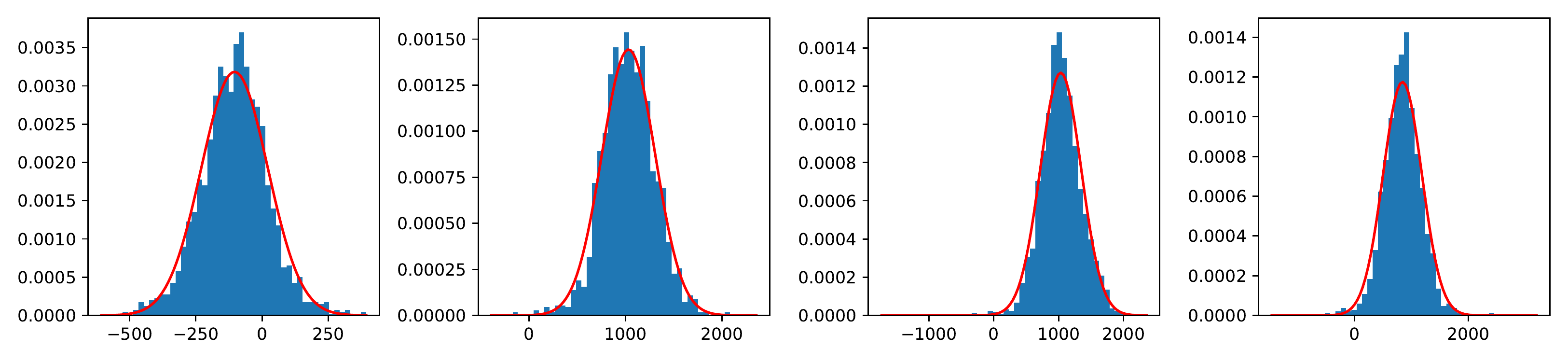}
    \caption{Histograms of models on Diamonds.}
    \label{fig:hist-diamonds}
\end{figure}

\begin{figure}
    \centering
    \includegraphics[width = \textwidth]{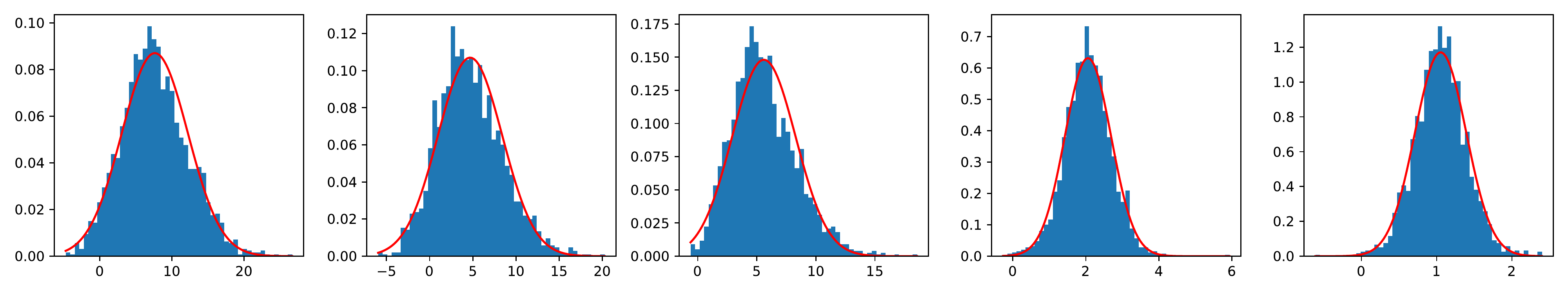}
    \caption{Histograms of models on NBA.}
    \label{fig:hist-nba}
\end{figure}

\begin{figure}
    \centering
    \includegraphics[width = \textwidth]{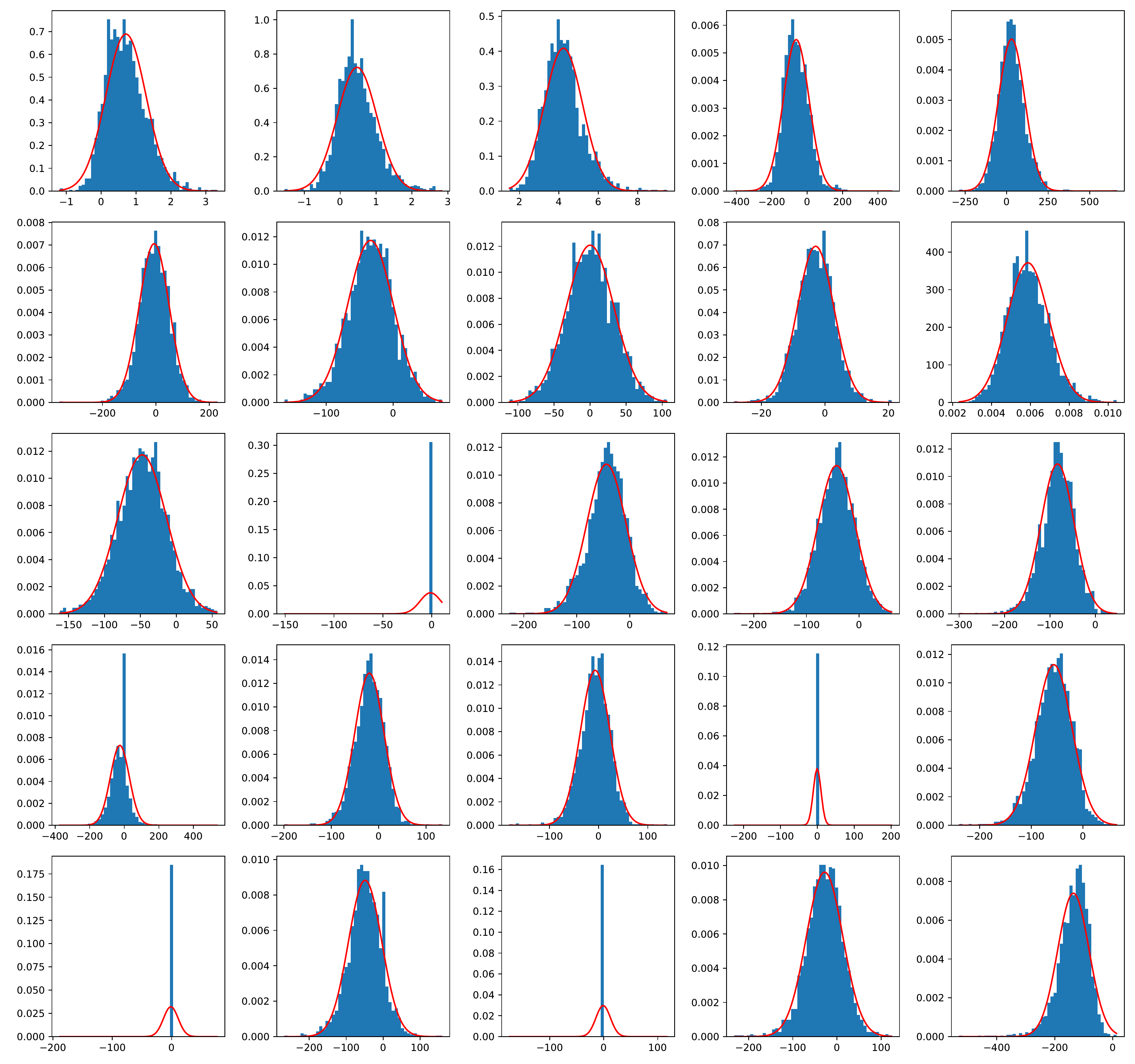}
    \caption{Histograms of models on Beijing.}
    \label{fig:hist-beijing}
\end{figure}

\begin{figure}
    \centering
    \includegraphics[width = \textwidth]{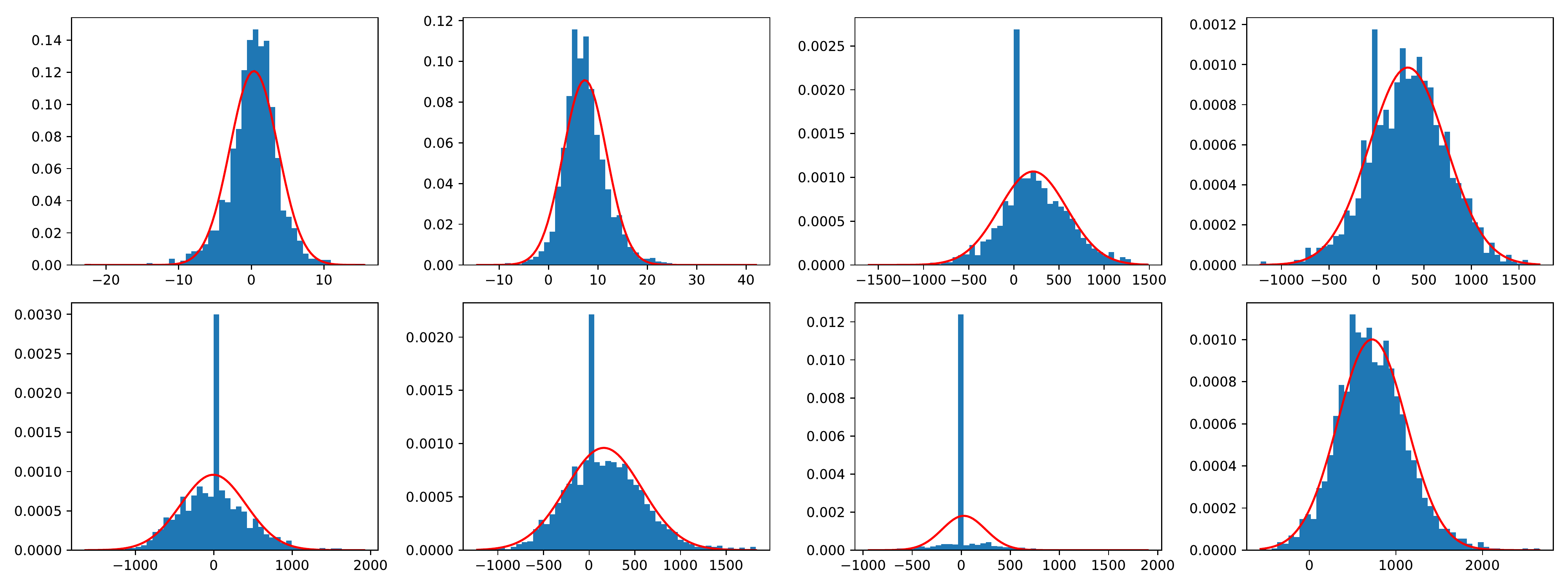}
    \caption{Histograms of models on Garbage.}
    \label{fig:hist-garbage}
\end{figure}

\begin{figure}
    \centering
    \includegraphics[width = \textwidth]{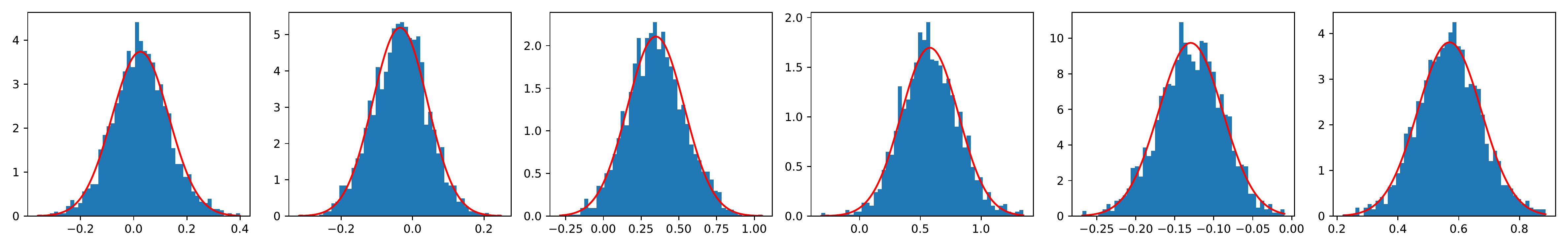}
    \includegraphics[width = \textwidth]{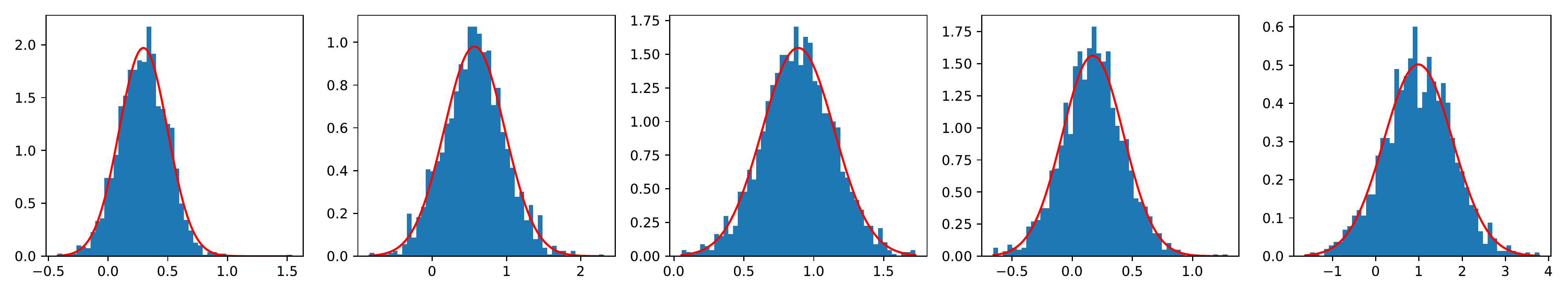}
    \caption{Histograms of models on MLB. }
    \label{fig:hist-mlb}
\end{figure}

\end{document}